\documentclass[11pt]{article}

\usepackage[utf8]{inputenc} % allow utf-8 input
\usepackage[T1]{fontenc}    % use 8-bit T1 fonts
\usepackage{hyperref}       % hyperlinks
\usepackage{url}            % simple URL typesetting
\usepackage{booktabs}       % professional-quality tables
\usepackage{amsfonts}       % blackboard math symbols
\usepackage{nicefrac}       % compact symbols for 1/2, etc.
\usepackage{microtype}      % microtypography
\usepackage{xcolor}         % colors
\usepackage{subcaption}
\usepackage[numbers]{natbib}  % Might cause problems with arxiv
\usepackage{fullpage}

\usepackage{tcolorbox}
\usepackage{caption}
\usepackage{enumitem}

\usepackage{liyang}

\graphicspath{{img/}}

\hypersetup{
    colorlinks=true,
    linkcolor=blue,
    urlcolor=cyan,
    }

\title{Harnessing the Power of Choices in Decision Tree Learning}

\newcommand*\samethanks[1][\value{footnote}]{\footnotemark[#1]}

\author{%
  Guy Blanc\thanks{Authors ordered alphabetically.} \vspace{3pt} \\
  {\sl Stanford } \vspace{2pt} \\
  \texttt{gblanc@stanford.edu} \\
  \And
  Jane Lange\samethanks \vspace{3pt} \\
  {\sl MIT} \vspace{2pt} \\
  \texttt{jlange@mit.edu} \\
  \And
  Chirag Pabbaraju\samethanks \vspace{3pt}\\
  {\sl Stanford} \vspace{2pt} \\
  \texttt{cpabbara@stanford.edu} \\
  \And 
  Colin Sullivan\samethanks \vspace{3pt}\\
  {\sl Stanford} \vspace{2pt} \\
  \texttt{colins26@stanford.edu} \\
  \And
  Li-Yang Tan\samethanks \vspace{3pt} \\
  {\sl Stanford}  \vspace{2pt}\\
  \texttt{lytan@stanford.edu} \\
  \And
  Mo Tiwari\samethanks \vspace{3pt} \\
  {\sl Stanford} \vspace{2pt} \\
  \texttt{motiwari@stanford.edu} \\
}

\newcommand{\topk}{\text{Top-}k}
\newcommand{\Top}[1]{\text{Top-}#1}
\newcommand{\maj}{\textnormal{Maj}}

\newcommand{\tribes}{\textnormal{Tribes}}

\def\colorful{0}

\ifnum\colorful=1
\newcommand{\violet}[1]{{\color{violet}{#1}}}

\newcommand{\gray}[1]{{\color{gray}{#1}}}

\fi
\ifnum\colorful=0
\newcommand{\violet}[1]{{{#1}}}

\newcommand{\gray}[1]{{{#1}}}

\fi

\renewcommand{\gray}[1]{}

\begin{document}

\maketitle

\newcommand{\lnote}[1]{\footnote{{\bf \color{blue}Li-Yang}: {#1}}}
\newcommand{\jnote}[1]{\footnote{{\bf \color{red}Jane}: {#1}}}
\newcommand{\cnote}[1]{\footnote{{\bf \color{green}Chirag}: {#1}}}
\newcommand{\gnote}[1]{\footnote{{\bf \color{violet}Guy}: {#1}}}
\begin{abstract}
We propose a simple generalization of standard and empirically successful decision tree learning algorithms such as ID3, C4.5, and CART.   These algorithms, which have been central to machine learning for decades, are greedy in nature: they grow a decision tree by iteratively splitting on the best attribute.  Our algorithm, Top-$k$, considers the $k$ best attributes as possible splits instead of just the single best attribute.% (i.e.~the special case of Top-$1$ captures ID3, C4.5, and CART). 

We demonstrate, theoretically and empirically, the power of this simple generalization.  We first prove a {\sl greediness hierarchy theorem} showing that for every $k\in \N$, Top-$(k+1)$  can be dramatically more powerful than Top-$k$: there are data distributions for which the former achieves accuracy $1-\eps$, whereas the latter only achieves accuracy $\frac1{2}+\eps$.  We then show, through extensive experiments, that Top-$k$ outperforms the two main approaches to decision tree learning: classic greedy algorithms and more recent ``optimal decision tree'' algorithms.  On one hand, Top-$k$ consistently enjoys significant accuracy gains over greedy algorithms across a wide range of benchmarks.  On the other hand, Top-$k$ is markedly more scalable than optimal decision tree algorithms and is able to handle dataset and feature set sizes that remain far beyond the reach of these algorithms.  

%\begin{sloppypar}
%The code to reproduce our results is available at 
%All of our results are reproducible via a one-line script available online at the following link: 
The code to reproduce our results: \texttt{https://github.com/SullivanC19/pydl8.5-topk}.
% \end{sloppypar}
\end{abstract}

%\lnote{Final check: Standardize Top-$k$ as name of algorithm, $h$ has depth of tree instead of $\rho$, theorem statements in the intro synced up with those in the body} 
\section{Introduction}

Decision trees are a fundamental workhorse in machine learning.  Their logical and hierarchical structure makes them easy to understand and their predictions easy to explain.  Decision trees are therefore the most canonical example of an interpretable model: in his influential survey~\citep{Bre01TwoCultures}, Breiman writes ``On interpretability, trees rate an A+''; much more recently, the survey~\cite{RCCHSZ22} lists decision tree optimization as the very first of 10 grand challenges for the field of interpretable machine learning.  Decision trees are also central to modern ensemble methods such as random forests~\citep{Bre01} and XGBoost \citep{CG16}, which achieve state-of-the-art accuracy for a wide range of tasks.

Greedy algorithms such as ID3~\citep{Qui86}, C4.5~\citep{Qui93}, and CART~\citep{BFSO84} have long been the standard approach to decision tree learning. These algorithms build a decision tree from labeled data in a top-down manner, growing the tree by iteratively splitting on the ``best'' attribute as measured with respect to a certain heuristic function (e.g., information gain).  Owing to their simplicity, these algorithms are highly efficient and scale gracefully to handle massive datasets and feature set sizes, and they continue to be widely employed in practice and enjoy significant empirical success. For the same reasons, these algorithms are also part of the standard curriculum in introductory machine learning and data science courses.  

The trees produced by these greedy algorithms are often reasonably accurate, but can nevertheless be suboptimal. There has therefore been a separate line of work, which we review in~\Cref{related},  on algorithms that optimize for accuracy and seek to produce {\sl optimally} accurate decision trees. These algorithms employ a variety of optimization techniques (including dynamic programming, integer programming, and SAT solvers) and are completely different from the simple greedy algorithms discussed above.  Since the problem of finding an optimal decision tree has long been known to be NP-hard~\citep{HR76}, {\sl any} algorithm must suffer from the inherent combinatorial explosion when the instance size becomes sufficiently large (unless P$=$NP).  Therefore, while this line of work has made great strides in improving the scalability of algorithms for optimal decision trees, dataset and feature set sizes in the high hundreds and thousands remain out of reach. 

This state of affairs raises a natural question: 
\begin{quote} 
{\sl Can we design decision tree learning algorithms that improve significantly on the accuracy of classic greedy algorithms and yet inherit their simplicity and scalability?}
\end{quote} 
In this work, we propose a new approach and make a case that provides a strong affirmative answer to the question above. Our work also opens up several new avenues for exploration in both the theory and practice of decision tree learning.  

\subsection{Our contributions} 

\subsubsection{Top-$k$: a simple and effective generalization of classic greedy decision tree algorithms}

%\paragraph{This work: A simple and effective generalization of classical decision tree algorithms}

We introduce an easily interpretable greediness parameter to the class of all greedy decision tree algorithms, a broad class that encompasses ID3, C4.5, and CART. This parameter, $k$, represents the number of features that the algorithm considers as candidate splits at each step. Setting $k=1$ recovers the fully greedy classical approaches, and increasing $k$ allows the practitioner to produce more accurate trees at the cost of only a mild training slowdown.  The focus of our work is on the regime where~$k$ is a small constant---preserving the efficiency and scalability of greedy algorithms is a primary objective of our work---although we mention here that by setting $k$ to be the dimension $d$, our algorithm produces an optimal tree.  Our overall framework can thus be viewed as interpolating between greedy algorithms at one extreme and ``optimal decision tree'' algorithms at the other, precisely the two main and previously disparate approaches to decision tree learning discussed above. 

We will now describe our framework.  A {\sl feature scoring function} $\mathcal{H}$ takes as input a dataset over $d$ binary features and a specific feature $i\in [d]$, and returns a value quantifying the ``desirability'' of this feature as the root of the tree.  The greedy algorithm corresponding to $\mathcal{H}$ selects as the root of the tree the feature that has the largest score under $\mathcal{H}$; our generalization will instead consider the $k$ features with the $k$ highest scores.  

\begin{definition}[Feature scoring function] A feature scoring function $\mathcal{H}$ takes as input a labeled dataset $S$ over a $d$-dimensional feature space, a feature $i\in [d]$, and returns a score $\nu_i \in [0,1]$.
\end{definition} 

See~\Cref{subsec:impurity} for a discussion of the feature scoring functions that correspond to standard greedy algorithms ID3, C4.5, and CART.  %For concreteness, we will focus on the feature scoring function of ID3\lnote{To what extent do we focus on ID3?}\gnote{I think all experiments are ID3? Theoretical results hold for all impurity based feature scoring functions, (though one could imagine feature scoring functions that are not impurity-based)}, though as mentioned, our framework is general and can be instantiated with any greedy decision tree algorithm. 
Pseudocode for $\topk$ is provided in \Cref{fig:pseudocode}. We note that from the perspective of interpretability, the trained model looks the same regardless of what $k$ is. During training, the algorithm considers more splits, but only one split is eventually used at each node.

\begin{figure}[ht]
%   \captionsetup{width=.9\linewidth}
\begin{tcolorbox}[colback = white,arc=1mm, boxrule=0.25mm]
\vspace{3pt}
$\textnormal{Top-}k(\mathcal{H},S, h)$:
\begin{itemize}[align=left]
	\item[\textbf{Given:}] A feature scoring function $\mathcal{H}$, a labeled sample set $S$ over $d$ dimensions, and depth budget $h$.
	\item[\textbf{Output:}] Decision tree of depth $h$ that approximately fits $S$. 
\end{itemize}
\begin{enumerate}
	\item If $h = 0$, or if every point in $S$ has the same label, return the constant function with the best accuracy w.r.t.~$S$.
	\item Otherwise, let $\mathcal{I}\sse [d]$ be the set of $k$ coordinates maximizing $\mathcal{H}(S, i)$.
	\item For each $i \in \mathcal{I}$, let $T_i$ be the tree with 
	\begin{align*} 
	\text{Root} &= x_i \\ 
	\text{Left subtree} &= \textnormal{Top-}k(\mathcal{H},S_{x_i = 0},h - 1) \\ 
	\text{Right subtree} &= \textnormal{Top-}k(\mathcal{H},S_{x_i = 1},h - 1), 
	\end{align*}
	where $S_{x_i = b}$ is the subset of points in $S$ where $x_i = b$.
	\item Return the $T_i$ with maximal accuracy with respect to $S$ among all choices of $i \in \mathcal{I}$.
\end{enumerate}
\end{tcolorbox}
\caption{The Top-$k$ algorithm. It can be instantiated with any feature scoring function $\mathcal{H}$, and when $k = 1$, recovers standard greedy algorithms such as ID3, C4.5, and CART. 
}
\label{fig:pseudocode}
\end{figure}

\subsubsection{Theoretical results on the power of Top-$k$}

The search space of Top-$(k+1)$ is larger than that of Top-$k$, and therefore its training accuracy is certainly at least as high.  The first question we consider is: is the test accuracy of Top-$(k+1)$ only marginally better than that of Top-$k$, or are there examples of data distributions for which even a single additional choice provably leads to huge gains in test accuracy?   Our first main theoretical result is a sharp {\sl greediness hierarchy theorem}, showing that this parameter can have dramatic impacts on accuracy, thereby illustrating its power: %\lnote{Changed $\rho$ to $h$ for depth/height of tree in the intro. To do: make this consistent throughout.} 
%\cnote{write Theorem 1 also wrt sample?}\gnote{Good point, changed}
\begin{theorem}[Greediness hierarchy theorem]
    \label{thm:k-hierarchy} For every $\eps > 0$, $k,h \in \N$, there is a data distribution $\mcD$ and sample size $n$ for which, with high probability over a random sample $\bS \sim \mcD^n$, \textnormal{Top}-$(k+1)$ achieves at least $1-\eps$ accuracy with a depth budget of $h$, but \textnormal{Top}-$k$ achieves at most $\frac1{2}+\eps$ accuracy with a depth budget of $h$.
\end{theorem}

All of our theoretical results, \Cref{thm:k-hierarchy,thm:k-hierarchy-general,thm:monotone-hierarchy-intro}, hold whenever the scoring function is an \emph{impurity-based heuristic}. This broad class includes the most popular scoring functions (see \Cref{subsec:impurity} for more details).
%\gnote{I added these two sentences. A reviewer brought up a good point that we need to specify the scoring function, so hopefully this addresses that.}
\Cref{thm:k-hierarchy} is a special case of a more general result that we show: for all $k  < K$, there are data distributions on which Top-$K$ achieves maximal accuracy gains over Top-$k$, even if Top-$k$ is allowed a larger depth budget:  

\begin{theorem}[Generalization of~\Cref{thm:k-hierarchy}]
\label{thm:k-hierarchy-general} 
For every $\eps > 0$, $k,K,h\in \N$ where $k < K$, there is a data distribution $\mcD$ and sample size $n$ for which, with high probability over a random sample $\bS \sim \mcD^n$, \textnormal{Top}-$K$ achieves at least $1-\eps$ accuracy with a depth budget of $h$, but \textnormal{Top}-$k$ achieves at most $\frac1{2}+\eps$ accuracy even with a depth budget of $h + (K-k-1)$.%\lnote{Check that this is indeed the right bound.} 
\end{theorem}

%\gnote{I think we get depth $d$ with top-$k$, vs depth $d + (k - k' - 1)$ for top-$k'$.}

The proof of~\Cref{thm:k-hierarchy-general} is simple and highlights the theoretical power of choices.  One downside, though, is that it is based on  data  distributions that are admittedly somewhat unnatural: the labeling function has embedded within it a function that is the XOR of certain features, and real-world datasets are unlikely to exhibit such adversarial structure.  To address this, we further prove that the power of choices is evident even for {\sl monotone} data distributions. We defer the definition of monotone data distributions to \Cref{subsec:monotone}.

\begin{theorem}
[Greediness hierarchy theorem for monotone data distributions]\label{thm:monotone-hierarchy-intro} 
For every $\eps > 0$, depth budget $h$, $K$ between $\tilde{\Omega}(h)$ and $\tilde{O}(h^2)$ and $k \leq K - h$, there is a monotone data distribution $\mcD$ and sample size $n$ for which, with high probability over a random sample $\bS \sim \mcD^n$, \Top{$K$} achieves at least $1 - \eps$ accuracy with a depth budget of $h$, but \Top{$k$} achieves at most $\frac1{2} + \eps$ accuracy with a depth budget of $h$.
\end{theorem}

Many real-world data distributions are monotone in nature, and relatedly, they are a common assumption and the subject of intensive study in learning theory.  Most relevant to this paper, recent theoretical work has identified monotone data distributions as a broad and natural class for which classical greedy decision tree algorithms (i.e., Top-$1$) provably succeed~\citep{BLT-ITCS,BLT-ICML}. \Cref{thm:monotone-hierarchy-intro} shows that even within this class, increasing the greediness parameter can lead to dramatic gains in accuracy.  Compared to~\Cref{thm:k-hierarchy-general}, the proof of~\Cref{thm:monotone-hierarchy-intro} is more technical  and involves the use of concepts from the Fourier analysis of boolean functions~\citep{ODBook}. %\lnote{Update if necessary} 

We note that a weaker version of~\Cref{thm:monotone-hierarchy-intro} is implicit in prior work: combining \citep[Theorem 7b]{BLT-ITCS} and \citep[Theorem 2]{BLQT21focs} yields the special case of \Cref{thm:monotone-hierarchy-intro} where $K = O(h^2)$ and $k = 1$. \Cref{thm:monotone-hierarchy-intro} is a significant strengthening as it allows for $k > 1$ and much smaller $K - k$.

%You might be concerned the function of \Cref{thm:k-hierarchy} was unnatural...

%\begin{theorem}[$k$-hierarchy theorem for monotone functions]

%\end{theorem}
%Mention that monotone functions are natural.

\subsubsection{Experimental results on the power of Top-$k$}

We provide extensive empirical validation of the effectiveness of Top-$k$ when trained on on real-world datasets, and provide an in-depth comparison with both standard greedy algorithms as well as optimal decision tree algorithms.

% \cnote{rewrite this}
We first compare the performance of $\topk$ for $k = 1,2,3,4,8,12,16$ (\Cref{fig:acc-topk-vs-top1}), and find that increasing $k$ does indeed provide a significant increase in test accuracy---in some cases, Top-$8$ already achieves accuracy comparable to the test accuracy attained by DL8.5~\citep{aglin2020learning}, an optimal decision tree algorithm. We further show, in \Cref{fig:time-topk-vs-top1,fig:optimal-trees}, that $\topk$ inherits the efficiency of popular greedy algorithms and scales much better than the state-of-the-art optimal decision tree algorithms MurTree and GOSDT \citep{GOSDT}. %\lnote{Add discussion of Figure 3 -- Chirag will handle.} 

Taken as a whole, our experiments demonstrate that $\topk$ provides a useful middle ground between greedy and optimal decision tree algorithms: it is significantly more accurate than greedy algorithms, but still fast enough to be practical on reasonably large datasets. See \Cref{sec:experiments} for an in-depth discussion of our experiments.
Finally, we emphasize the benefits afforded by the simplicity of Top-$k$. Standard greedy algorithms (i.e.~Top-$1$) are widely employed and easily accessible. Introducing the parameter $k$ requires modifying only a tiny amount of source code and gives the practitioner a new lever to control.  Our experiments and theoretical results demonstrate the utility of this simple lever.

\section{Related work}
\label{related}

\paragraph{Provable guarantees and limitations of greedy decision tree algorithms.} A long and fruitful line of work seeks to develop a rigorous understanding of the performances of greedy decision tree learning algorithms such as ID3, C4.5, and CART and to place their empirical success on firm theoretical footing~\citep{KM96,Kea96,DKM96,BDM19,BDM20,BLT-ITCS,BLT-ICML,BLQT21}.  These works identify feature and distributional assumptions under which these algorithms provably succeed; they also highlight the {\sl limitations} of these algorithms by pointing out settings in which they provably fail.  Our work complements this line of work by showing, theoretically and empirically, how these algorithms can be further improved with a simple new parameter while preserving their efficiency and scalability.  

\paragraph{The work of~\cite{BLQT21focs}.}  Recent work of Blanc, Lange, Qiao, and Tan~also highlights the power of choices in decision tree learning.  However, they operate within a stylized theoretical setting. First, they consider a specific scoring function that is based on a notion of {\sl influence} of features, and crucially, computing these scores requires {\sl query access} to the target function (rather than from random labeled samples as is the case in practice).  Furthermore, their results only hold with respect to the uniform distribution.  These are strong assumptions that limit the practical relevance of their results.  In contrast, a primary focus of this work is to be closely aligned with practice, and in particular, our framework captures and generalizes the standard greedy algorithms used in practice.

\paragraph{Optimal decision trees.} Motivated in part by the surge of interest in interpretable machine learning and the highly interpretable nature of decision trees, there have been numerous works on learning {\sl optimal} decision trees~\citep{BD17,VZ17,VZ19,AAV19,ZMPNK20,VNPQS20,NIPMS18,Ave20,JM20,NF07,NF10,HRS19,GOSDT,MURTREE}.
As mentioned in the introduction, this is an NP-complete problem~\citep{HR76}---indeed, it is NP-hard to find even an approximately optimal decision tree~\citep{Sie08,AH08,ABFKP08}.  Due to the fundamental intractability of this problem, even highly optimized versions of algorithms are unlikely to match the scalability of standard greedy algorithms. That said, these works implement a variety of optimizations that allow them to build optimal decision trees for many real world datasets when the dataset and feature sizes are in the hundreds and the desired depth is small ($\leq 5$).

Finally, another related line of work is that of {\sl soft} decision trees \citep{irsoy2012soft, tanno2019adaptive}. These works use gradient-based methods to learn soft splits at each internal node. We believe that one key advantage of our work over these soft trees is in interpretability. With $\topk$, since the splits are hard (and not soft), to understand the classification of a test point, it is sufficient to look at only one root-to-leaf path, as opposed to a weighted combination across many. 

\section{The $\topk$ algorithm} 
\label{sec:algorithm}

\subsection{Background and context: Impurity-based algorithms} 
\label{subsec:impurity}
Greedy decision tree learning algorithms like ID3, C4.5 and CART are all instantiations of $\topk$ in \Cref{fig:pseudocode} with $k=1$ and an appropriate choice of \violet{the} feature-scoring function $\mathcal{H}$. Those three algorithms all used {\sl impurity-based heuristics} as their feature-scoring function: 
\begin{definition}[Impurity-based heuristic]
    \label{def:impurity}
    An \emph{impurity function} $\mathcal{G}:[0,1] \to [0,1]$ is a function that is concave, symmetric about $0.5$, and satisfies $\mathcal{G}(0) = \mathcal{G}(1) = 0$ and $\mathcal{G}(0.5) = 1$. A feature-scoring function $\mathcal{H}$ is an \emph{impurity-based heuristic}, if there is some impurity function $\mathcal{G}$ for which: 
    \begin{align*}
        \mathcal{H}(S, i) &= \mathcal{G}\left(\Ex_{\bx,\by \sim S}[\by]\right) 
         - \Prx_{\bx,\by \sim S}[\bx_i = 0]\cdot\mathcal{G}\left(\Ex_{\bx,\by \sim S}[\by\mid \bx_i = 0]\right) \\&- \Prx_{\bx,\by \sim S}[\bx_i = 1]\cdot\mathcal{G}\left(\Ex_{\bx,\by \sim S}[\by\mid \bx_i = 1]\right)
    \end{align*}
    where in each of the above, $(\bx,\by)$ are a uniformly random point from within $S$.
\end{definition}
%Due to the concavity of $\mathcal{G}$, impurity-based heuristics give a high score to features that have large correlation with the label (i.e. when $\Ex_{\bx,\by \sim S}[\by\mid \bx_i = 0]$ and $\Ex_{\bx,\by \sim S}[\by\mid \bx_i = 1]$ are far apart). 
Common examples for the impurity function include the binary entropy function $\mathcal{G}(p)=-p\log_2(p) -(1-p)\log_2(1-p)$ (used by ID3 and C4.5), the Gini index $\mathcal{G}(p)=4p(1-p)$ (used by CART), and the function $\mathcal{G}(p)=2\sqrt{p(1-p)}$ (proposed and analyzed in~\cite{KM99}). We refer the reader to \cite{KM99} for a theoretical comparison, and \cite{DKM96} for an experimental comparison, of these impurity-based heuristics.% applied to \emph{greedy} algorithms ($\Top{1}$).

Our experiments focus on binary entropy being the impurity measure, but our theoretical results apply to $\Top{k}$ instantiated with \emph{any} impurity-based heuristic.

\gray{The feature-scoring function $\mathcal{H}$ is determined by a suitable heuristic for an ``impurity'' function $\mathcal{G}:[0,1] \to [0,1]$, which is such that it is constrained to be concave and symmetric around $0.5$. It also has the property that $\mathcal{G}(0) = \mathcal{G}(1) = 0$ (this corresponds to the ``pure node'' case) and $\mathcal{G}(0.5)=1$ (this corresponds to the ``maximally-impure node'' case). Given an impurity function $\mathcal{G}$, the feature scoring function $\mathcal{H}(S,i)$ on a labelled dataset $S = \{(x^{(j)}, y^{(j)})\}_{j=1}^n$, where each $x^{(j)} \in \{0,1\}^d, y^{(j)} \in \{0,1\}$, scores a feature $i$ as follows:
\begin{align*}
    \mathcal{H}(S, i) = \mathcal{G}(\Pr_{S}[y=1]) - (\Pr_S[x_i=0]\mathcal{G}(\Pr_{S_{x_i=0}}[y=1]) + \Pr_S[x_i=1]\mathcal{G}(\Pr_{S_{x_i=1}}[y=1]).
\end{align*}
As concrete examples, the impurity function used in ID3 and C4.5 is the binary entropy function $\mathcal{G}(p)=-p\log_2(p) -(1-p)\log_2(1-p)$, while CART uses the Gini index $\mathcal{G}(p)=4p(1-p)$. Additionally, Kearns and Mansour \citep{KM99} also analyze and prove guarantees for the impurity function $\mathcal{G}(p)=2\sqrt{p(1-p)}$. We note again that $\topk$ can be instantiated with any of these popularly used feature-scoring functions.}

% Give examples of impurity functions and how that leads to feature scoring functions. 

% Instantiate Figure 1 with different $\mathcal{H}$s and tell what each is.

\subsection{\violet{Basic theoretical properties} of the Top-$k$ algorithm}
\label{subsec:topk-analysis}
\paragraph{Running time.} The key behavioral aspect in which $\topk$ differs from greedy algorithms is that it is less greedy when trying to determine which coordinate to query. This naturally increases the running time of $\topk$, but that increase is fairly mild. More concretely, suppose $\topk$ is run on a dataset $S$ with $n$ points. We can then easily derive the following bound on the running time of $\topk$, where $\mathcal{H}(S, i)$ is assumed to take $O(n)$ time to evaluate (as it does for all impurity-based heuristics).
\begin{claim}
\label{claim:running-time}
The running time of $\topk(\mathcal{H}, S, h)$ is $O((2k)^h \cdot nd) $.
\end{claim}
\begin{proof}
    \violet{Let $T_h$ be the number of recursive calls made by $\Top{k}(\mathcal{H}, S, h)$. Then, we have the simple recurrence relation $T_h = 2kT_{h-1}$, where $T_{0}=1$. Solving this recurrence gives $T_h = (2k)^h$. Each recursive call takes $O(nd)$ time, where the bottleneck is scoring each of the $d$ features.}
    \gray{Referring to the pseudocode in \Cref{fig:pseudocode}, note that the set $\mathcal{I}$ in Step 2 can be identified in time $O(nd)$ by running $\mathcal{H}$ on all $d$ coordinates. Also, if we denote the total number of invoked recursive calls to $\topk$ by $T_h$, we have the simple recurrence relation $T_h = 2kT_{h-1}$, where $T_{0}=1$. Solving this recurrence, we have that that the total number of recursive calls $T_h = (2k)^h$, and hence the total running time of $\topk(\mathcal{H}, S,h)$ is $O((2k)^h\poly(n, d)).$}
\end{proof}
We note that any decision tree algorithm, including fast greedy algorithms such as ID3, C4.5, and CART, has runtime that scales exponentially with the depth $h$. The size of a depth-$h$ tree can be $2^h$, and this is of course a lower bound on the runtime as the algorithm needs to output such a tree. In contrast with greedy algorithms (for which $k=1$), $\topk$ incurs an additional $k^h$ cost in running time. As mentioned earlier, in practice, we are primarily concerned with fitting small decision trees (e.g., $h = 5$) to the data, as this allows for explainable predictions. In this setting, the additional $k^h$ cost (for small constant $k$) is inexpensive, as confirmed by our experiments.

\paragraph{The search space of $\Top{k}$:} We state and prove a simple claim that $\Top{k}$ returns the \emph{best} tree within its search space.%\gnote{change to sample rather than distribution}

\begin{definition}[Search space of $\Top{k}$]
    Given a sample $S$ and integers $h,k$,
    we use $\mathcal{T}_{k,h,S}$ to refer to all trees in the search space of $\Top{k}$.
    Specifically, if $h = 0$, this contains all trees with a height of zero (the constant $0$ and constant $1$ trees). For $h \geq 1$, and $\mathcal{I} \subseteq [d]$ being the $k$ coordinates with maximal score, this contains all trees with a root of $x_i$, left subtree in $\mathcal{T}_{k,h-1, S_{x_i = 0}}$ and right subtree in $\mathcal{T}_{k,h-1, S_{x_i = 1}}$ for some $i \in \mathcal{I}$.
\end{definition}

\begin{lemma}[$\Top{k}$ chooses the most accurate tree in its search space]
    \label{lem:best-in-search-space}
    For any sample $S$ and integers $h, k$, let $T$ be the output of $\Top{k}$ with a depth budget of $h$ on $S$. Then
    \begin{align*}
        \Prx_{\bx,\by \sim S}[T(\bx) = \by] = \max_{T' \in \mathcal{T}_{k,h,S}}\left( \Prx_{\bx,\by \sim S}[T'(\bx) = \by]\right).
    \end{align*}
\end{lemma}

We refer the reader to \Cref{appsec:searchspace} for the proof of this lemma.

% \paragraph{Parallelizability.} We observe that $\topk$ is largely amenable to parallelization. Observe that each of the $2k$ recursive calls in Step 3 of the pseudocode can be assigned to a separate processor. This means that given $(2k)^h$ processors, a carefully engineered implementation of $\topk$ can enjoy a parallel running time of $O(hnd)$, which is also the parallel running time of standard greedy algorithms.

%Analysis, bounds on runtime, etc.  State that the algorithm is highly parallelizable. 

\section{Theoretical bounds on the power of choices}
\label{sec:theory}

We refer the reader to the \Cref{sec:theory-appendix} for most of the setup and notation. For now, we briefly mention a small amount of notation relevant to this section:  we use \textbf{bold font} (e.g. $\bx$) to denote random variables. We also use bold font to indicate \emph{stochastic functions} which output a random variable. For example,
\begin{equation*}
    \boldf(x) \coloneqq \begin{cases}
    x & \text{with probability $\lfrac{1}{2}$}\\
    -x & \text{with probability $\lfrac{1}{2}$}
    \end{cases}
\end{equation*}
is the stochastic function that returns either the identity or its negation with equal probability. To define the data distributions of \Cref{thm:k-hierarchy-general,thm:monotone-hierarchy-intro}, we will give a distribution over the domain, $X$ and the stochastic function that provides the label given an element of the domain. %Throughout this section, the domain is $X \coloneqq \zo^d$ and the distribution over it is uniform.

%\lnote{Add brief justification/citation for the fact that when the distribution is uniform, ID3 splits on the feature wit highest correlation. In fact, we can say that this is the case for all impurity based heuristics with concave impurity function etc.  Should be in the ITCS or ICML paper.}\gnote{This is \Cref{fact:score-correlation-monotone}}

%\gray{In this section, we will assume for simplicity that $\topk$ has exact access to $f$ and can exactly compute the heuristic scores $\mathcal{H}(f,i)$ for each $x_i$, rather than estimating them from a sample. Furthermore, we will let the distribution over inputs $x$ be uniform over $\bits^d$.}

\paragraph{Intuition for proof of greediness hierarchy theorem}
To construct a distribution which \textnormal{Top}-$k$ fits poorly and \textnormal{Top}-$(k+1)$ fits well, we will partition features into two groups: one group consisting of features with medium correlation to the labels and another group consisting of features with high correlation when taken all together but low correlation otherwise.
Since the correlation of features in the former group is larger than that of the latter group unless all features from the latter group are considered, both algorithms will prioritize features from the former group.
However, if the groups are sized correctly, then \textnormal{Top}-$(k+1)$ will consider splitting on all features from the latter group, whereas \textnormal{Top}-$k$ will not.
As a result, \textnormal{Top}-$(k+1)$ will output a decision tree with higher accuracy.

\subsection{Proof of \Cref{thm:k-hierarchy-general}}
\label{subsec:nonmonotone}
For each depth budget $h$ and search branching factor $K$,
we will define a hard distribution $\mathcal{D}_{h,K}$ that is learnable to high accuracy by $\Top{K}$ with a depth of $h$,
but not by $\topk$ with a depth of $h'$ for any $h' < h + K - k$.
This distribution will be over $\zo^d \times \zo$, where $d = h + K - 1$.
The marginal distribution over $\zo^d$ is uniform,
and the distribution over $\zo$ conditioned on a setting of the $d$ features is given by the stochastic function $\boldf_{h,K}(x)$. All of the results of this section (\Cref{thm:k-hierarchy-general,thm:monotone-hierarchy-intro}) hold when the feature scoring function is \emph{any} impurity-based heuristic.

\paragraph{Description of $\boldf_{h,K}(x)$.} %\gnote{Can we use $\boldf$ for stochastic functions?}
Partition $x$ into two sets of variables, $x^{(1)}$ of size $h$ and $x^{(2)}$ of size $K-1$.
Let $\boldf_{h, K}(x)$ be the randomized function defined as follows:
\[\boldf_{h,K}(x) = \begin{cases}
\mathrm{Par}_{h}(x^{(1)}) &\text{with probability }1 - \eps \\
x^{(2)}_i \sim \mathrm{Unif}[x^{(2)}] &\text{with probability }\eps,
% x^{(2)}_i \text{ chosen uniformly at random from }x^{(2)} &\text{with probability }\eps.
\end{cases}\] 
where $\mathrm{Unif}[x^{(2)}]$ denotes the uniform distribution on $x^{(2)}$. $\mathrm{Par}_h(x^{(1)})$ is the parity function, whose formal definition can be found in \Cref{sec:theory-appendix}.

The proof of \Cref{thm:k-hierarchy-general} is divided into two parts.
First, we prove that when the data distribution is $\mathcal{D}_{h, K}$,
$\Top{K}$ succeeds in building a high accuracy tree with a depth budget of $h$.
Then, we show that $\Top{k}$ fails and builds a tree with low accuracy,
even given a depth budget of $h + (K - k - 1)$.

\begin{lemma}[$\Top{K}$ succeeds]
\label{lem:top-k-succeed-nonmonotone}
The accuracy of $\Top{K}$ with a depth of $h$ on $\mathcal{D}_{h,K}$ is at least $1 - \eps$.
\end{lemma}

\begin{lemma}[$\topk$ fails]
\label{lem:top-k-fail-nonmonotone}
The accuracy of $\topk$ with a depth of $h'$ on $\mathcal{D}_{h,K}$ is at most $(1/2 + \eps)$ for any $h' < h + K - k$.
\end{lemma}
Proofs of both these lemmas are deferred to \Cref{sec:theory-appendix}. \Cref{thm:k-hierarchy-general} then follows directly from these two lemmas.

\subsection{Proof of \Cref{thm:monotone-hierarchy-intro}} 
\label{subsec:monotone}

In this section, we overview the proof \Cref{thm:monotone-hierarchy-intro}. Some of the proofs are deferred to \Cref{subsec:monotone-appendix}.

% \begin{theorem}[Greediness hierarchy theorem for monotone distributions]
% \label{thm:k-hierarchy-monotone}
%     For every $\eps > 0$, depth budget $h$, $K$ between $\Omega(h \log h)$ and $O(h^2 / (\log h)^2)$ and $k \leq K - h$, there is a monotone data distribution on which \Top{$K$} achieves at least $1 - \eps$ accuracy with a depth budget of $h$, but \Top{$k$} achieves at most $0.5 + \eps$ accuracy with a depth budget of $h$.
% \end{theorem} 

% \gray{Combining \cite[Theorem 7b]{BLT-ITCS} and \cite[Theorem 2]{BLQT21focs} implies the above where $K = O(h^2)$ and $k = 1$. \Cref{thm:k-hierarchy-monotone} represents a significant strengthening of that result, as it allows for $k > 1$ and $K - k$ quite a bit smaller.}\gnote{This could maybe be in the intro instead}

Before proving \Cref{thm:monotone-hierarchy-intro}, we formalize the concept of monotonicity. For simplicity, we assume the domain is the Boolean cube, $\zo^d$, and use the partial ordering $x \preceq x'$ iff $x_i \leq x_i'$ for each $i \in [d]$; however, the below definition easily extends to the domain being any partially ordered set. 
\begin{definition}[Monotone]
    A stochastic function, $\boldf:\zo^d \to \zo$, is \emph{monotone} if, for any $x, x' \in \zo^d$ where $x \preceq x'$, $\Ex[\boldf(x)] \leq \Ex[\boldf(x')]$. A data distribution, $\mathcal{D}$ over $\zo^d \times \zo$ is said to be monotone if the corresponding stochastic function, $\boldf(x)$ returning $(\by \mid \bx = x)$ where $(\bx,\by) \sim \mathcal{D}$, is monotone.
\end{definition}

To construct the data distribution of \Cref{thm:monotone-hierarchy-intro}, we will combine monotone functions, Majority and Tribes, commonly used in the analysis of Boolean functions due to their extremal properties. See \Cref{subsec:monotone-appendix} for their definitions and useful properties. Let $d = h + K-1$, and the distribution over the domain be uniform over $\zo^d$. Given some $x \in \zo^d$, we use $x^{(1)}$ to refer to the first $h$ coordinates of $x$ and $x^{(2)}$ the other $K-1$ coordinates. This data distribution is labeled by the stochastic function $\boldf$ given below.
\begin{equation*}
    \boldf(x) \coloneqq \begin{cases}
        \tribes_{h}(x^{(1)}) & \text{with probability $1 - \eps$} \\
        \maj_{K-1}(x^{(2)}) & \text{with probability $\eps$.}
    \end{cases}
\end{equation*}
Clearly $\boldf$ is monotone as it is the mixture of two monotone functions. Throughout this subsection, we'll use $\mathcal{D}_{h,K}$ to refer to the data distribution over $\zo^d \times \zo$ where to sample $(\bx,\by) \sim \mcD$, we first draw $\bx \sim \zo^d$ uniformly and then $\by$ from $\boldf(\bx)$. The proof of \Cref{thm:monotone-hierarchy-intro} is a direct consequence of the following two Lemmas, both of which we prove in \Cref{subsec:monotone-appendix}.

\begin{lemma}[\Top{$K$} succeeds]
    \label{lem:top-K-succeed-monotone}
    On the data distribution $\mathcal{D}_{h,K}$, $\Top{K}$ with a depth budget of $h$ achieves at least $1 - \eps$ accuracy.
\end{lemma}

\begin{lemma}[\Top{$k$} fails]
    \label{lem:top-k-fail-monotone}
    On the data distribution $\mathcal{D}_{h,K}$, $\Top{k}$ with a depth budget of $h$ achieves at most $\frac{1}{2} + \eps$ accuracy.
\end{lemma}

% \begin{theorem}[Generalization of~\Cref{thm:k-hierarchy}]
% \label{thm:k-hierarchy-general} 
% For every $k,K,h\in \N$ where $k < K$, there is a data distribution on which \textnormal{Top}-$K$ achieves at least 99\% accuracy with a depth budget of $h$, but \textnormal{Top}-$k$ achieves at most 51\% accuracy even with a depth budget of $h + (K-k-1)$.\lnote{Check that this is indeed the right bound.} 
% \end{theorem}
\section{Experiments}
\label{sec:experiments}

\paragraph{Setup for experiments.} At all places, the $\Top{1}$ tree that we compare to is that given by \texttt{scikit-learn}~\citep{scikit-learn}, which according to their documentation\footnote{\href{https://scikit-learn.org/stable/modules/tree.html\#tree-algorithms-id3-c4-5-c5-0-and-cart}{https://scikit-learn.org/stable/modules/tree.html\#tree-algorithms-id3-c4-5-c5-0-and-cart}}, is an optimized version of CART. We run experiments on a variety of datasets from the UCI Machine Learning Repository~\citep{UCI} (numerical as well as categorical features) having a size in the thousands and having $\approx 50-300$ features \violet{after binarization}. There were \violet{$\approx 100$} datasets meeting these criteria, and we took a random subset of $20$ such datasets. We binarize all the datasets -- for categorical datasets, we convert every categorical feature that can take on (say) $\ell$ values into $\ell$ binary features. For numerical datasets, we sort and compute thresholds for each numerical attribute, so that the total number of binary features is $\approx 100$.  A detailed description of the datasets is given in \Cref{appsec:dataset-details}.

We build decision trees corresponding to binary entropy as the impurity measure $\mathcal{H}$. In order to leverage existing engineering optimizations from state-of-the-art optimal decision tree implementations, we implement the $\topk$ algorithm given in \Cref{fig:pseudocode} via simple modifications to the PyDL8.5 \citep{aglin2020learning, aglin2021pydl8} codebase\footnote{\href{https://github.com/aia-uclouvain/pydl8.5}{https://github.com/aia-uclouvain/pydl8.5}}. Details about this are provided in \Cref{appsec:implementation-details}. Our implementation of the $\topk$ algorithm and other technical details for the experiments are available at \texttt{https://github.com/SullivanC19/pydl8.5-topk}.

%randomly sampled from the UCI Machine Learning Repository~\cite{UCI} (numerical as well as categorical features) having a size in the thousands and having $\approx 50-300$ features. We binarize all the datasets - for categorical datasets, we convert every categorical feature that can take on (say) $l$ values into $l$ binary features. For numerical datasets, we sort and compute thresholds for each numerical attribute for an appropriate number of thresholds. A detailed description of the datasets is given in \Cref{appsec:dataset-details}. For a fixed depth, we build decision trees corresponding to both gini and entropy as the impurity measure $\mathcal{H}$, and report numbers for whichever of the two performed better on the test set. A simple implementation of the $\topk$ algorithm and other technical details for the experiments will be made publicly available.

\subsection{Key experimental findings}

\paragraph{\violet{Small increments of $k$ yield significant accuracy gains.}}%$\topk$ gets better test accuracy than $\Top{1}$}

\begin{figure*}[t!]
    \centering
	\includegraphics[width=1\linewidth]{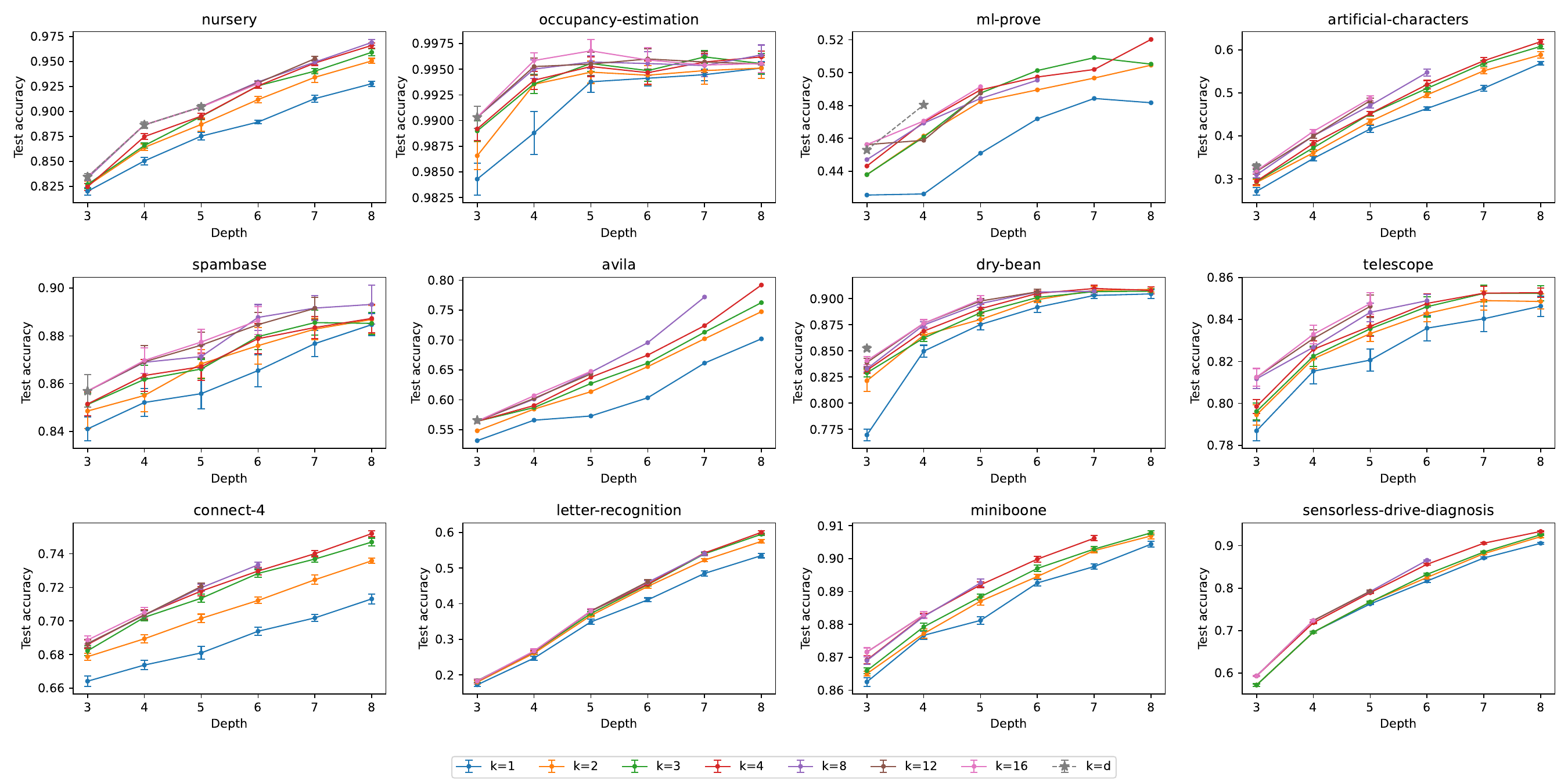}  
	\caption{Test accuracy comparison between $\topk$ for various values of $k$. We can see that Top-$(k+1)$ generally obtains higher accuracy than Top-$k$, and in some cases (e.g., nursery), $\Top{8/16}$'s accuracy is even comparable to the optimal tree ($\Top{d}$). Missing points in the plots correspond to settings that did not terminate within a sufficiently large time limit. All plots are averaged over 10 random train-test splits (except avila and ml-prove that have pre-specified splits) \violet{with confidence intervals plotted for 2 standard deviations}.}
	\label{fig:acc-topk-vs-top1}
\end{figure*}
%\lnote{Include brief interpretation of figures in captions.  E.g. something like ``We see that Top-$k+1$ generally obtains higher accuracy than Top-$k$ for $k=1,2,3$, and Top-$4$'s accuracy is even comparable to MurTree's."}
Since the search space of $\topk$ is a superset of that of $\Top{1}$ for any $k > 1$, the training accuracy of $\topk$ is guaranteed to be larger. The primary objective in this experiment is to show that $\topk$ can outperform $\Top{1}$ in terms of test accuracy as well. \Cref{fig:acc-topk-vs-top1} shows the results for $\Top{1}$ versus $\topk$ for $k=2,3,4,8,12,16,d$. Each plot is a different dataset, where on the x-axis, we plot the depth of the learned decision tree, and on the y-axis, we plot the test accuracy. Note that $k=d$ corresponds to the DL8.5 optimal decision tree.
%We also plot the test accuracy of an optimal decision tree (MurTree) in each plot as an additional point of reference.\footnote{This number is agnostic to what algorithm is computing the optimal tree, and hence we only compute it for MurTree, since it scales up much better than GOSDT (furthermore, the GOSDT tree is not exactly optimal unless the regularization coefficient is set to 0).}
We can clearly observe that the test accuracy increases as $k$ increases---in some cases, the gain is $> 5\%$ (absolute). \violet{Furthermore, for (smaller) datasets like nursery, for which we were able to run $k=d$, the accuracy of Top-$8/16$ is already very close to that of the optimal tree}.  

Lastly, since $\topk$ invests more computation towards fitting a better tree on the training set, its training time is naturally longer than $\Top{1}$.  \violet{However, \Cref{fig:time-topk-vs-top1} in \Cref{appsec:training-time-comparison-top-1}, which plots the training time, shows that the slowdown is mild.} % \red{links to Colin's experiment about running time, should be done before supplement deadline}

\paragraph{$\topk$ scales much better than optimal decision tree algorithms.}
% Despite having an optimality certificate, optimal decision tree algorithms suffer increased running time compared to $\Top{1}$.
Optimal decision tree algorithms suffer from poor runtime scaling.
%In particular, if we want to learn a tree of depth 6 (say), even heavily optimized state-of-the-art optimal trees like MurTree and GOSDT cannot scale beyond $\approx 200$ features on large datasets.
We empirically demonstrate that, in comparison, $\topk$ has a significantly better scaling in training time. Our experiments are identical to those in Figures 14 and 15 in the GOSDT paper \citep{GOSDT}, where two notions of scalability are considered. In the first experiment, we fix the number of samples and gradually increase the number of features to train the decision tree. In the second experiment, we include all the features, but gradually increase the number of training samples. The dataset we use is the FICO~\citep{FICO} dataset, which has a total of 1000 samples with 1407 binary features. We plot the training time (in seconds) versus number of features/samples for optimal decision tree algorithms (MurTree, GOSDT) and $\topk$ in \Cref{fig:optimal-trees}. We do this for depth $=4,5,6$ (for GOSDT, the regularization coefficient $\lambda$ is set to $2^{-\text{depth}}$). We observe that the training time for both MurTree and GOSDT increases dramatically compared to $\topk$, in both experiments. In particular, for depth $=5$, both MurTree and GOSDT were unable to build a tree on 300 features within the time limit of 10 minutes, while $\Top{16}$ completed execution even with all 1407 features. Similarly, in the latter experiment, GOSDT/MurTree were unable to build a depth-5 tree on 150 samples within the time limit, while $\Top{16}$ comfortably finished execution even on 1000 samples. These experiments demonstrates the scalability issues with optimal tree algorithms. Coupled with the accuracy gains seen in the previous experiment, $\topk$ can thus be seen as achieving a more favorable tradeoff between training time and accuracy.

We note, however, that various optimization have been proposed to allow optimal decision tree algorithms to scale to larger datasets.
For example, a more recent version of GOSDT has integrated a guessing strategy using reference ensembles which guides the binning of continuous features, tree size, and search \cite{GOSDT-guessing}.
Many of these optimizations are generally applicable across optimal tree algorithms and could be combined with $\topk$ for further improvement in performance.

%\footnote{We also ran an experiment comparing accuracy with Soft Decision Trees \citep{irsoy2012soft}. We found that their code took significantly longer to train (often 1-2 orders of magnitude) than both Top-$3$ and Top-$4$ trained to the same depth. For accuracy, it seems that each method has data sets where it performs better; however, when Soft Decision Trees has more accuracy, it is typically by a very small amount. In contrast, in the other cases, $\topk$ often has a drastic accuracy gain.}\gnote{This footnote about soft trees looks a bit out of place to me. I would vote to remove it and if a reviewer complains we can add it back.}

% For completeness, we also provide accuracy plots for this experiment in \Cref{appsec:optimal-trees-further-plots}. \cnote{Yet to include samples test accuracy plot in appendix.} \cnote{Do we want to include these at all?}
% \begin{itemize}
%     \item graph with top1/2/3 and gosdt showing gradual increase in slope
%     \item graph with x-axis being number of features and y-axis being the runtime
%     \item graph with x-axis being number of samples and y-axis being the runtime
% \end{itemize}

\begin{figure*}[t]
	\begin{subfigure}{.33\textwidth}
	\centering
	% include first image
	\includegraphics[width=1\linewidth]{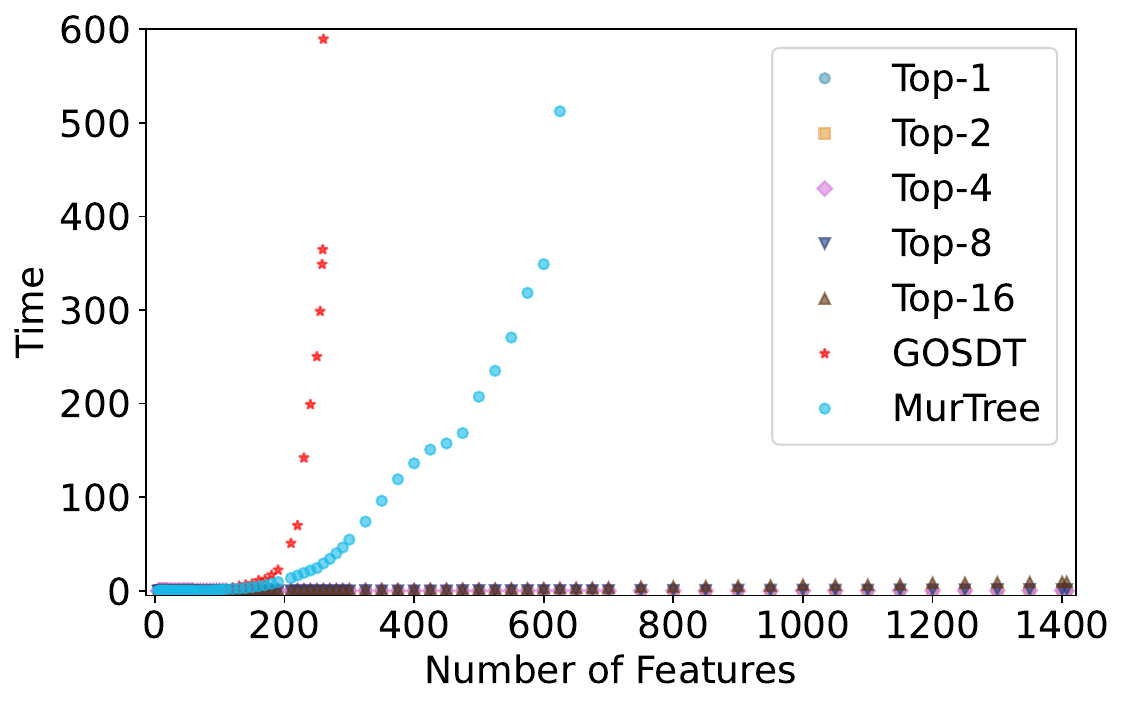}  
	\caption{Depth $=4$}
	\label{fig:optimal-trees-features-depth-4}
\end{subfigure}
	\begin{subfigure}{.33\textwidth}
	\centering
	% include third image
	\includegraphics[width=1\linewidth]{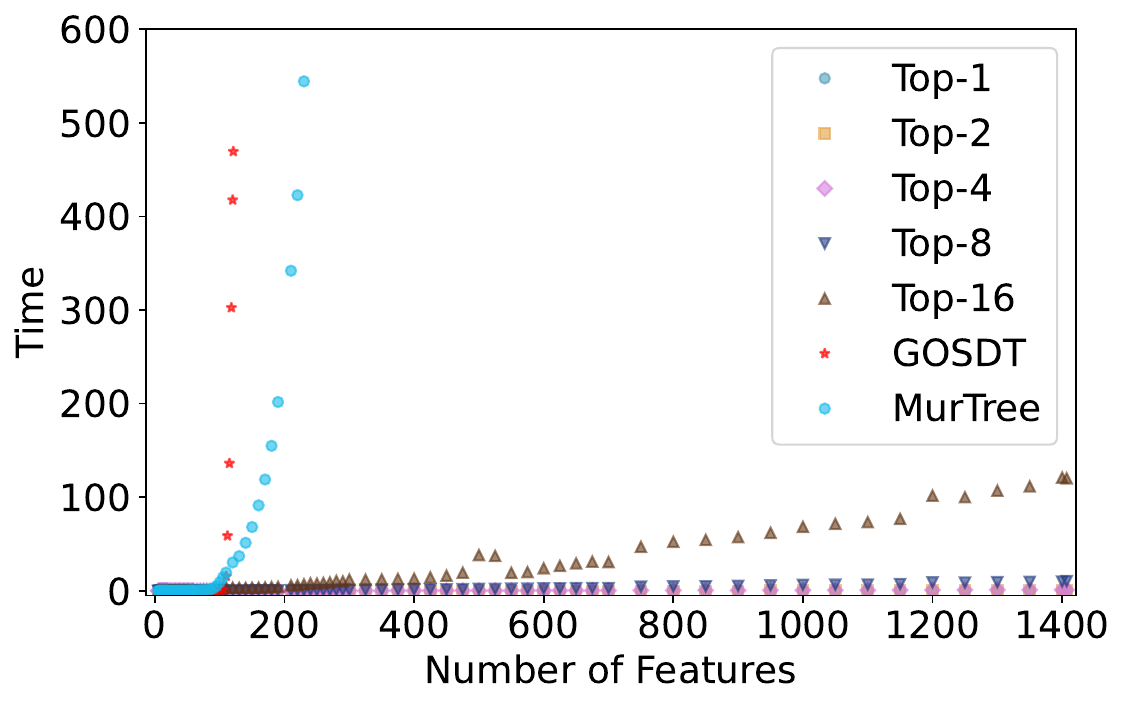}  
	\caption{Depth $=5$}
	\label{fig:optimal-trees-feautures-depth-5}
\end{subfigure}
	\begin{subfigure}{.33\textwidth}
	\centering
	% include first image
	\includegraphics[width=1\linewidth]{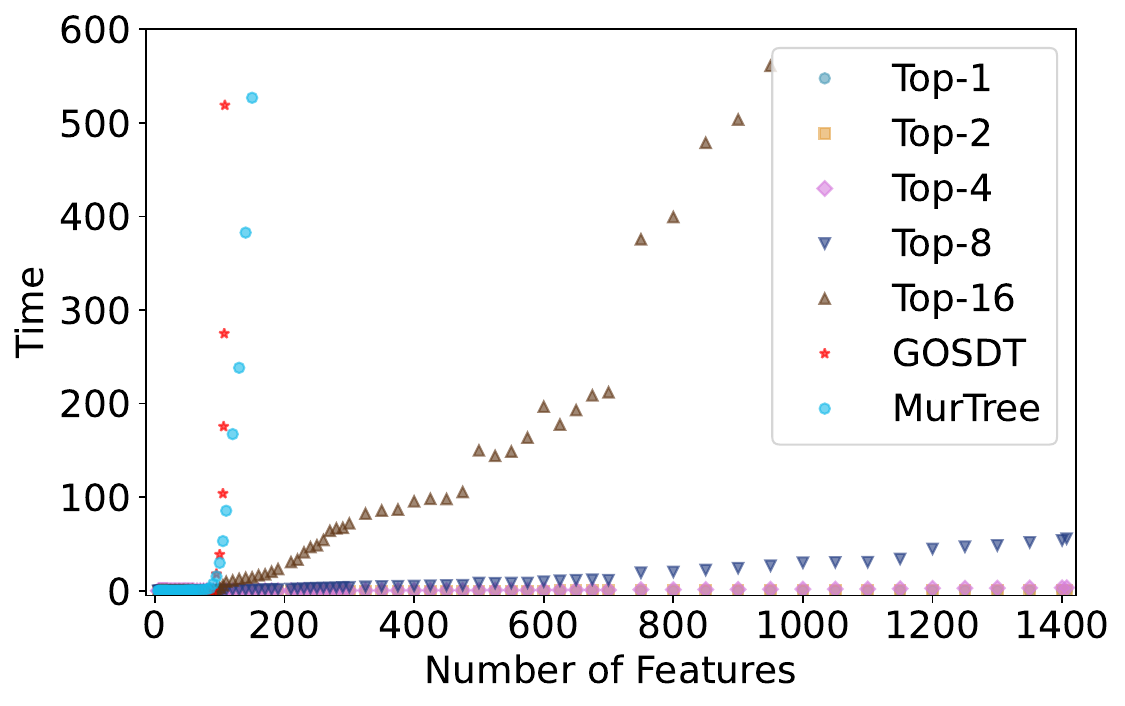}  
	\caption{Depth $=6$}
	\label{fig:optimal-trees-features-depth-6}
\end{subfigure}
\newline
	\begin{subfigure}{.33\textwidth}
	\centering
	% include first image
	\includegraphics[width=1\linewidth]{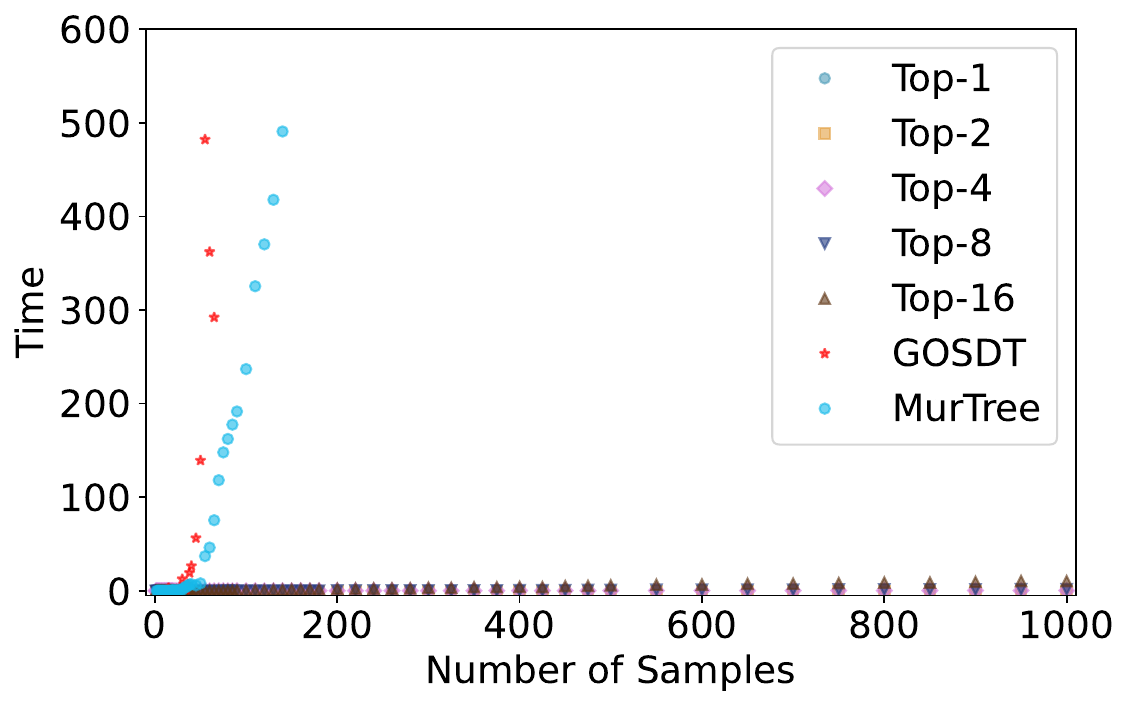}  
	\caption{Depth $=4$}
	\label{fig:optimal-trees-samples-depth-4}
\end{subfigure}
	\begin{subfigure}{.33\textwidth}
	\centering
	% include third image
	\includegraphics[width=1\linewidth]{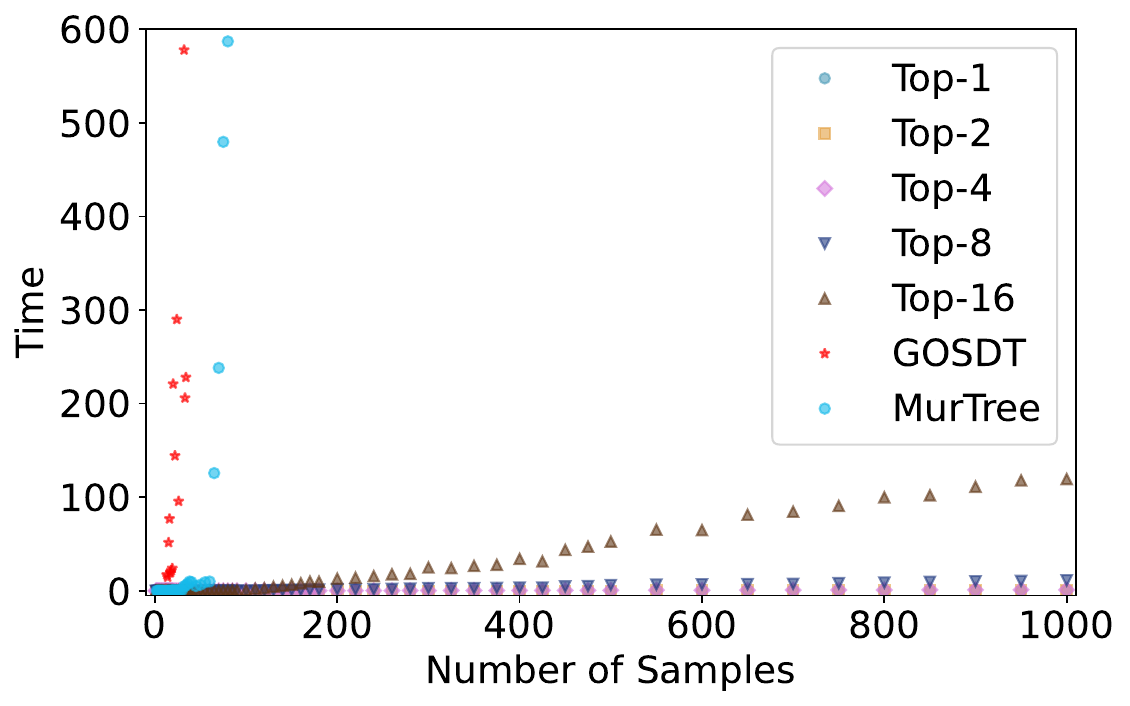}  
	\caption{Depth $=5$}
	\label{fig:optimal-trees-samples-depth-5}
\end{subfigure}
	\begin{subfigure}{.33\textwidth}
	\centering
	% include first image
	\includegraphics[width=1\linewidth]{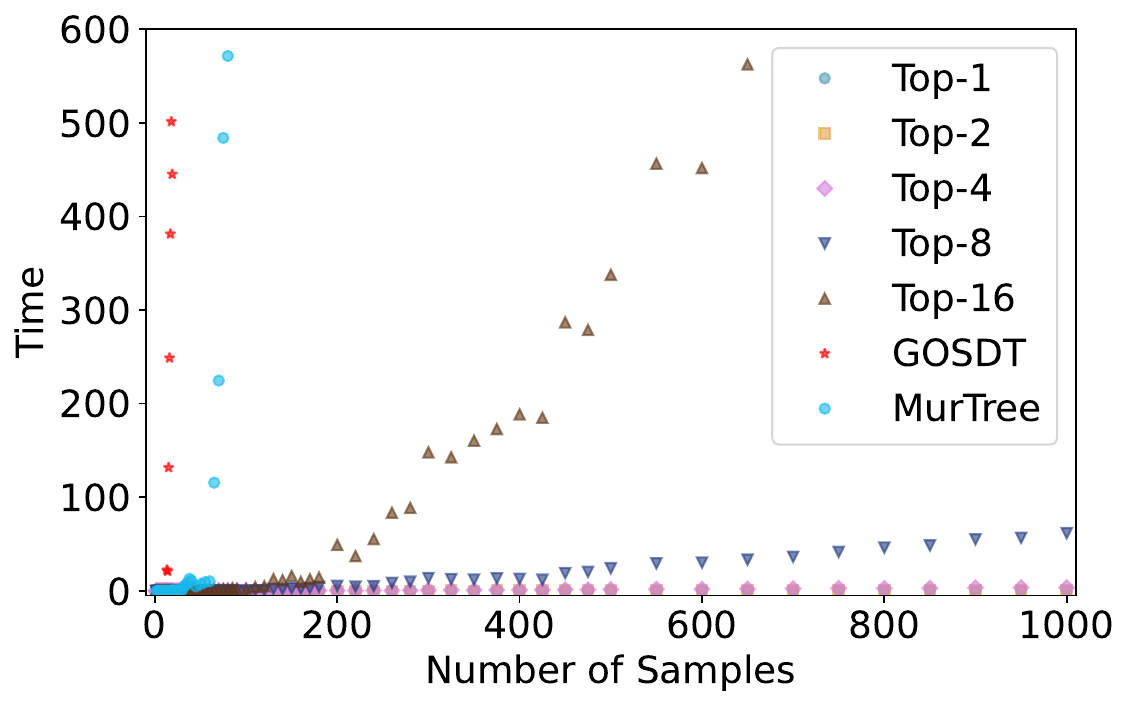}  
	\caption{Depth $=6$}
	\label{fig:optimal-trees-samples-depth-6}
\end{subfigure}
	\caption{Training time comparison between $\topk$ and optimal tree algorithms. As the number of features/samples increases, both GOSDT and MurTree scale poorly compared to $\topk$, and beyond a threshold, do not complete execution within the time limit.}
	\label{fig:optimal-trees}
\end{figure*}

\begin{figure*}[t!]
    \centering
	\includegraphics[width=1\linewidth]{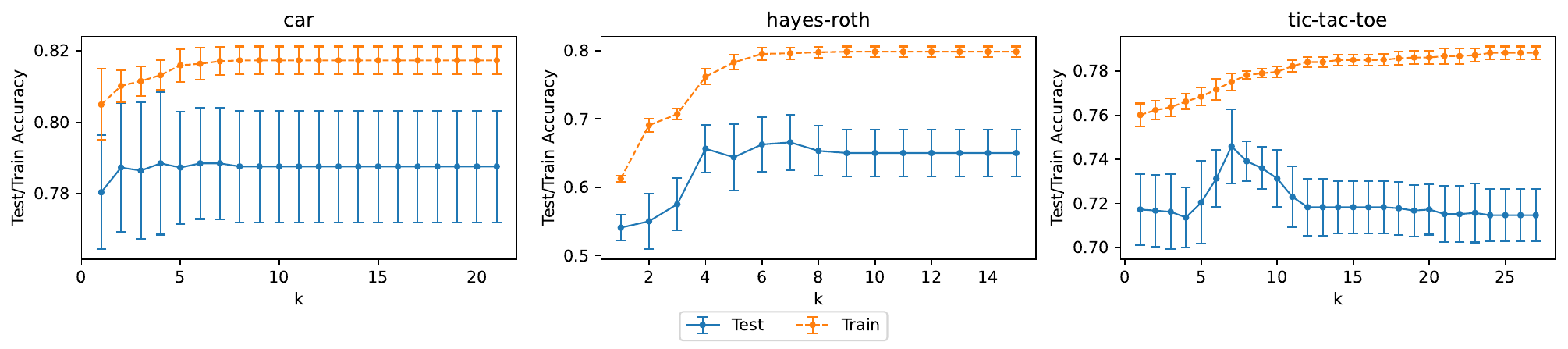}  
	\caption{Test accuracy plateaus for large $k$. All runs averaged over 10 random train-test splits with maximum depth fixed to 3.}
	\label{fig:increasing-k-acc-plateau}
\end{figure*}

% \paragraph{Accuracy given a runtime budget}
% (maybe include)
% \begin{itemize}
%     \item Colors are different choices of $k$
%     \item $x$-axis is runtime (maybe on log-scale? tbd)
%     \item $y$-axis is accuracy
%     \item Should be able to see what depth is by counting the dots from the $y$-axis
% \end{itemize}

\paragraph{Increasing $k$ beyond a point does not improve test accuracy.}%\gnote{I rewrote this? Chirag do you like?}\cnote{yes looks good}
In our experiments above, we ran $\Top{k}$ only till $k=16$: in \Cref{fig:increasing-k-acc-plateau}, we show that increasing $k$ to very large values, which increases runtime, often does not improve test accuracy, and in some cases, may even \textit{hurt} due to overfitting. For 3 datasets -- car, hayes-roth and tic-tac-toe -- we plot train and test error as a function of $k$. Naturally, the train accuracy monotonically increases with $k$ in each plot. However, for both car and hayes-roth, we can observe that the test accuracy first increases and then plateaus. Interestingly, for tic-tac-toe, the test accuracy first increases and then \textit{decreases} as we increase $k$. These experiments demonstrate that selecting too large of a $k$, as optimal decision tree algorithms do, is a waste of computational resources and can even hurt test accuracy via overfitting.

%The experiment consists of increasing $k$ from $1$ all the way to $d$, and plotting both the train and test accuracies in a single plot. We do this for 3 datasets: car, hayes-roth and tic-tac-toe; \Cref{fig:increasing-k-acc-plateau} has the results. Naturally, the train accuracy monotonically increases with $k$ in each plot. However, for both car and hayes-roth, we can observe that the test accuracy first increases and then plateaus. Interestingly, for tic-tac-toe, the test accuracy first increases and then \textit{decreases} as we increase $k$. \violet{These experiments demonstrate} 

%Thus, it is important to tune $k$ properly\gnote{Maybe just end this sentence here}, to reap the benefits of both higher test accuracy and efficient training.

\section{Conclusion}

We have shown how popular and empirically successful greedy decision tree learning algorithms can be improved with {\sl the power of choices}: our generalization, Top-$k$, considers the~$k$ best features as candidate splits instead of just the single best one.  As our theoretical and empirical results demonstrate, this simple generalization is powerful and enables significant accuracy gains while preserving the efficiency and scalability of standard greedy algorithms. Indeed, we find it surprising that such a simple generalization has not been considered before.

There is much more to be explored and understood, both theoretically and empirically; we list here a few concrete directions that we find particularly exciting and promising.  First, we suspect that power of choices affords more advantages over greedy algorithms than just accuracy gains.  For example, an avenue for future work is to show that the trees grown by Top-$k$ are more {\sl robust to noise}.  Second, are there principled approaches to the automatic selection of the greediness parameter~$k$?  Can the optimal choice be inferred from a few examples or learned over time?  This opens up the possibility of new connections to machine-learned advice and algorithms with predictions~\citep{MV20}, an area that has seen a surge of interest in recent years.  Finally, as mentioned in the introduction,  standard greedy decision tree algorithms are at the very heart of modern tree-based ensemble methods such as XGBoost and random forests.  A natural next step is to combine these algorithms with Top-$k$ and further extend the power of choices to these settings.

 \section*{Acknowledgements}

We thank the NeurIPS reviewers and AC for their detailed and helpful feedback. 

Guy and Li-Yang are supported by NSF awards 1942123, 2211237, 2224246 and a Google Research Scholar award. Jane is supported by NSF Graduate Research Fellowship under Grant No.~2141064, NSF Awards CCF-2006664, DMS-2022448, and Microsoft. Mo is supported by a Stanford Interdisciplinary Graduate Fellowship and a Stanford Data Science Scholarship. Chirag is supported by Moses Charikar and Greg Valiant's Simons Investigator Awards.

%\newpage
\bibliography{neurips2023_conference}{}
\bibliographystyle{alpha} % Not sure if we should change this to something else

\newpage
\appendix

\newpage
\section{Proofs deferred from \Cref{sec:algorithm}}
\label{appsec:searchspace}
\begin{proof}[Proof of \Cref{lem:best-in-search-space}]
    By induction: When $h = 0$, the only trees in the search space are the constant $0$ and constant $1$ functions. $\Top{k}$ returns which of these two trees is the most accurate.
    
    When $h \geq 1$, let $T'$ be a tree with maximal accuracy within $\mathcal{T}_{k,h,S}$. As $T'$ is in the search space, its root must be one of the $k$ coordinates with maximal score which form the candidate set $\mathcal{I}$. 
    
    For each coordinate $i \in \mathcal{I}$, the candidate tree $T_i$ satisfies 
    \[\Prx_{\bx,\by \sim S}[T_i(\bx) \ne \by] = \Prx_{\bx \sim S}[x_i = 0]\Prx_{\bx,\by \sim S}[T_{i0}(\bx) \ne \by] + \Prx_{\bx\sim S}[x_i = 1]\Prx_{\bx,\by \sim S}[T_{i1}(\bx) \ne \by],\]
    where $T_{i0}$ and $T_{i1}$ are the left and right subtrees of $T_i$ respectively. 
    Each of $T_{i0}$ and $T_{i1}$ is an output of $\topk$ with a depth budget of $h-1$.
    We assume as the inductive hypothesis that each of these trees minimizes error among all trees in $\mathcal{T}_{k, h-1, S_{x_i = 0}}$ and $\mathcal{T}_{k, h-1, S_{x_i = 1}}$ respectively;
    therefore the candidate $T_i$ minimizes error among all trees in $\mathcal{T}_{k, h, S}$ that have $x_i$ at the root. 
    Since $\topk$ chooses the most accurate of the $T_i$'s, 
    it follows that the chosen tree minimizes error among all trees in $\mathcal{T}_{k, h, S}$.
\end{proof}

\section{Proofs deferred from \Cref{sec:theory}}
\label{sec:theory-appendix}

\paragraph{Setup and notation:} We use $\Ind[\cdot]$ for the indicator function, and $[d]$ to refer to the set $\{1, \ldots, d\}$.

%It will be helpful for us to define \emph{restrictions}, which naturally specify the subset of the domain corresponding to a single path through a decision tree. A restriction is specified by a partial function, $\alpha:[d]\rightharpoonup\{0,1\}$ indicating which (if any) coordinates are fixed to specific values. That restriction corresponds to the subset of the domain $X_\alpha \subseteq X$ where $x \in X_{\alpha}$ if, for all $i$ in the domain of the partial function $\alpha$, $x_i = \alpha(i)$. We'll also use restricted distributions $\mathcal{D}_\alpha$, where to sample $(\bx, \by) \sim \mathcal{D}_{\alpha}$, we sample $(\bx,\by) \sim \mathcal{D}$ conditioned on $\bx \in X_{\alpha}$.
%\gnote{Not sure if we will need restrictions are not}

For brevity, we will make two simplifying assumptions about $\Top{k}$:
\begin{enumerate}
    \item We will assume $\Top{k}$ builds \emph{non-redundant} trees, meaning on every root-to-leaf path, each coordinate is queried at most once. This is easy to enforce in the pseudocode: at each step, the algorithm can track a set $Q$ of the coordinates already queried along this path, and pick the top-$k$ coordinates according to the feature score function among $[d] \setminus Q$. For brevity, we do not include that modification to the pseudocode in \Cref{fig:pseudocode}.
    % \item We assume, roughly speaking, that $\Top{k}$ has access to an infinite-sized sample. More precisely, whenever $\Top{k}$ needs to compute an expectation over its sample (to determine the set of $k$ coordinates maximizing the scoring function, or to decide which constant function to put at a leaf), we replace the the empirical expectation with the population expectation. Using standard techniques, if the sample is large enough, these expectations will concentrate and our results still hold with high probability.

    % As discussed earlier (\textbf{reference}), when the input distribution is uniform on $\zo^d$, the $k$ variables selected by $\Top{k}$ to maximize the impurity-based heuristic are the $k$ coordinates $i \in [d]$ maximizing the \emph{empirical} correlation, $\Ex_{(\bx ,\by) \sim S}[\bx_i \by]$. Instead, for simplicity, we will assume $\Top{k}$ chooses the $k$ coordinates maximizing the \emph{population} correlation, $\Ex_{(\bx ,\by) \sim \mathcal{D}}[\bx_i \by]$. Using standard techniques, if the sample is large enough, the empirical correlation will concentrate around the population correlation and our results still hold. Similarly, we will assume that in the base case, $\Top{k}$ returns the constant function with best accuracy w.r.t. 
    \item We assume that $\Top{k}$ always build \emph{complete} trees (i.e every root-to-leaf path has depth exactly $h$). This is without loss of generality, as whenever $\Top{k}$ stops early, it does so because it has already achieved perfect accuracy on that path.
\end{enumerate}

\violet{Furthermore, $\Top{k}$ only uses the information in its sample in two ways: first, it uses the sample to compute the feature scoring function $\mcH(S, i)$. Second, when $h=0$, it uses the sample to determine whether the constant $0$ or constant $1$ fits the sample better. Both of these are ``statistical queries" \cite{Kea98SQ}, meaning the interaction the algorithm receives from the sample is simply the expectations $\Ex_{(\bx,\by) \sim S}[\phi_i(\bx,\by)]$ where $\phi_1,\ldots, \phi_t:\zo^{d+1} \to [0,1]$ are a sequence of queries. For any $\eps,\delta > 0$, by a standard concentration argument and union bound, for large enough sample size $n \geq n(\eps,\delta)$,
\begin{equation*}
    \Prx_{\bS \sim \mcD^n}\bracket*{\max_{i \in [t]}{\abs*{\Ex_{(\bx,\by)\sim \bS}[\phi_i(\bx,\by)]-\Ex_{(\bx,\by)\sim \mcD}[\phi_i(\bx,\by)]} \geq \eps}} \leq \delta.
\end{equation*}
Therefore, for sufficiently large sample size, we are free to assume that when the algorithm computes $\Ex_{(\bx,\by)\sim \bS}[\phi_i(\bx,\by)]$, it receives $\Ex_{(\bx,\by)\sim \mcD}[\phi_i(\bx,\by)]$ with high probability. This is a standard argument (c.f. \cite{KM96}), and so we will work directly with expectations from $\mcD$ in our proof to ease notation.}

\violet{Recall that \Cref{thm:k-hierarchy,thm:k-hierarchy-general,thm:monotone-hierarchy-intro} hold whenever the feature scoring function is an impurity-based heuristic.}%\gnote{Previously, using an impurity based heuristic was an ``assumption" that only appeared in our proof. I changed it so it's part of the theorem itself.} 
As our data distribution is uniform on the input, we are able to use the following fact and simultaneously prove results for all impurity-based heuristic:

\begin{fact}[Proposition 7.7 of \cite{BLT-ITCS}]
    \label{fact:score-correlation-monotone}
    If the scoring function is \emph{any} impurity-based heuristic, and the data distribution is uniform over inputs ($\bx$ is uniform when $(\bx,\by) \sim \mathcal{D}$), then the score of a coordinate $i$ is monotone increasing with its correlation with the label, $\Ex_{(\bx,\by)\sim \mathcal{D}}[\bx_i \by]$.
\end{fact}
Intuitively, \Cref{fact:score-correlation-monotone} means that, when analyzing $\Top{k}$ on uniform data distributions, we are free to replace the ``$k$ coordinates with largest scores" with the ``$k$ coordinates with largest correlations."

\subsection{Proofs deferred from \Cref{subsec:nonmonotone}}

The stochastic function $\boldf_{h,K}$ used throughout \Cref{lem:top-k-succeed-nonmonotone} and \Cref{lem:top-k-fail-nonmonotone} combines a function that outputs a random one of $k$ features with the $h$-wise parity function.
\begin{definition}[Parity]
    The \emph{parity} function of $\ell$ variables, indicated by $\mathrm{Par}_\ell: \zo^\ell \to \zo$, returns
    \begin{equation*}
        \mathrm{Par}_\ell(x) \coloneqq \bigg(\sum_{i \in [\ell]} x_i\bigg) \mod 2.
    \end{equation*}
\end{definition}

% \brown{
% \begin{proposition}[Parity has $0$ correlation with each feature]
%     \label{prop:parity-corr}
%     For any $\ell \geq 2$ and $i \in [\ell]$,
%     \begin{equation*}
%         \Ex_{\bx \sim \zo^\ell}[\bx_i \cdot \mathrm{Par}_\ell(\bx)] = 0.
%     \end{equation*}
% \end{proposition}
% \begin{proof}
    
% \end{proof}

% }

\begin{fact}[Computing any function with a complete tree]
    \label{fact:complete-tree}
    Let $f:\zo^d \to \zo$ be any function that only depends on the first $h$ variables, meaning there is some $g:\zo^h \to \zo$ such that:
    \begin{equation*}
        f(x) = g(x_{[1:h]})
    \end{equation*}
    for all $x \in \zo^d$. Let $T$ be any non-redundant complete tree of depth-$h$ in which every internal node is one of the first $h$ coordinates. Then, there is a way to label the leaves of $T$ such that $T$ exactly computes $f$.
\end{fact}
\begin{proof}
    Since $T$ is non-redundant, each coordinate is queried at most once on each root-to-leaf path. $T$ is complete and depth-$h$, so each of the first $h$ coordinates must be queried \emph{exactly} once on each root-to-leaf path. Therefore, each leaf of $T$ corresponds to exactly one way to set the first $k$ coordinates of $x$. If the leaf is labeled by the output of $g$ given those first $k$ coordinates, $T$ will exactly compute $f$.
\end{proof}

\begin{proof}[Proof of \Cref{lem:top-k-succeed-nonmonotone}]
The function $\mathrm{Par}_{h}(x^{(1)})$ is a $(1-\eps)$-approximation to $f$,
so it suffices to show that the depth-$h$ tree for $\mathrm{Par}_{h}(x^{(1)})$ is within the search space of $\Top{K}$ when run to a depth of $h$. Then we can apply \Cref{lem:best-in-search-space} to reach the desired result. 

There are only $K-1$ variables not in $x^{(1)}$, 
so each set of $K$ candidate variables must contain some variable in $x^{(1)}$.
Since $\Top{K}$ is non-redundant, this must be a variable
that has not yet been queried higher in the tree.
Thus, at every step $\Top{K}$ will always try a candidate variable that reduces the number of relevant $x^{(1)}$-variables by 1.
It follows that the complete nonadaptive tree of depth $h$, 
containing all the variables of $x^{(1)}$, 
is within the search space,
so by \Cref{fact:complete-tree} there is a tree in the search space that computes $\mathrm{Par}_h(x^{(1)})$ exactly.
Then the accuracy of the output must be at least the total accuracy of this tree, which is $(1-\eps)$.
\end{proof}

\begin{proof}[Proof of \Cref{lem:top-k-fail-nonmonotone}]
Conditioned on any setting of $< k$ variables, for any variable $x_i$ in $x^{(2)}$, $\E[f(x)x_i] \ge 1/k$.
Similarly, for any variable $x_j$ in $x^{(1)}$, $\E[f(x)x_j] = 0$. 
By \Cref{fact:score-correlation-monotone},
at every node the variables of $x^{(2)}$ that have not yet been queried
all rank ahead of the variables of $x^{(1)}$.%\gnote{This only applies to variables in $x^{(1)}$ that have not already been queried}
Thus, if at most $K - k$ variables have already been queried, 
the remaining $k$ most-correlated candidates will all be from $x^{(2)}$,
so no variable in $x^{(1)}$ will be considered.
Thus, at least $K - k$ variables from $x^{(2)}$ will be placed in every path. \\

Since the depth budget $h'$ is smaller than $h + K - k$
and at least $K - k$ variables from $x^{(2)}$ are placed in every path,
no path can contain all of the $h$ variables of $x^{(1)}$.
The value of $\mathrm{Par}_{h}(x^{(1)})$ is 0 with probability 1/2 and 1 with probability 1/2 
conditioned on the values of any set of variables smaller than $h$. 
Therefore, the tree built by $\topk$ cannot achieve accuracy better than 1/2 on the parity portion of the function
(and thus have accuracy better than $(1/2  + \eps)$ overall).

\end{proof}

\subsection{Proofs deferred from \Cref{subsec:monotone}}
\label{subsec:monotone-appendix}

The data distribution showing the accuracy separation between $\Top{K}$ and $\Top{k}$ is formed by combining the Majority and Tribes functions.
\begin{definition}[Majority]
    The \emph{majority} function of $\ell$ variables, indicated by $\maj_\ell: \zo^\ell \to \zo$, returns
    \begin{equation*}
        \maj_\ell(x) \coloneqq \Ind[\text{at least half of $x$'s coordinates are $1$}].
    \end{equation*}
\end{definition}

\begin{definition}[Tribes]
    For any input length $\ell$, let $w$ be the largest integer such that $(1 - 2^{-w})^{\ell/w} \leq 1/2$. For $x \in \zo^\ell$, let $x^{(1)}$ be the first $w$ coordinates, $x^{(2)}$, the second $w$, and so on. $\tribes_\ell$ is defined as
    \begin{equation*}
        \tribes_{\ell}(x) \coloneqq (x^{(1)}_1 \land \cdots \land x^{(1)}_w) \lor \cdots \lor (x^{(t)}_{1} \land \cdots \land x^{(t)}_{w}) \quad \quad \text{where } t \coloneqq \left \lfloor \frac{\ell}{w} \right\rfloor.
    \end{equation*}
\end{definition}
For our purposes, it is sufficient to know a few simple properties about $\tribes$. These are all proven in \citep[\S 4.2]{ODBook}.
\begin{fact}[Properties of $\tribes$]\leavevmode
    \label{fact:tribes-properties}
    \begin{enumerate}
        \item $\tribes_\ell$ is monotone.
        \item $\tribes_\ell$ is nearly balanced:
        \begin{equation*}
            \Ex_{\bx \sim \zo^\ell}[\tribes_\ell(\bx)] = \frac{1}{2} \pm o(1)
        \end{equation*}
        where the $o(1)$ term goes to $0$ as $\ell$ goes to $\infty$.
        \item All variables in $\tribes_\ell$ have small correlation: For each $i \in [\ell]$,
        \begin{equation*}
            \Cov_{\bx \sim \zo^\ell}[\bx_i, \tribes_\ell(\bx)] = O\left(\frac{\log \ell}{\ell}\right).
        \end{equation*}
    \end{enumerate}
\end{fact}
Indeed, the famous KKL inequality implies that any function with the first and second property has a variable with correlation at least $\Omega(\log \ell/\ell)$ \citep{KKL88}. Our construction uses $\tribes$ exactly because it has the minimum correlations among functions with the above properties (up to constants). In contrast, we use Majority because its correlations are as \emph{large} as possible, which will ``trick" $\Top{k}$ into building a bad tree.

With the above definitions in-hand, we are able to provide proofs of the following two lemmas:

\begin{proof}[Proof of \Cref{lem:top-K-succeed-monotone}]
    This proof is very similar to that of \Cref{lem:top-k-succeed-nonmonotone}: Once again, we observe the tree computing $(x \mapsto \tribes_{h}(x^{(1)}))$ has at least $1 - \eps$ accuracy with respect to $\mathcal{D}_{h, K}$. By \Cref{lem:best-in-search-space}, it is sufficient to prove such a tree is in the search space.
    
    By \Cref{fact:complete-tree}, any non-redundant complete tree of depth $h$ that only queries the first $h$ coordinates of its input will compute the function $(x \mapsto \tribes_{h}(x^{(1)}))$ whenever the leaves are appropriately labeled. Therefore, we only need to prove such a tree is in the search space $\mathcal{T}_{K,h,\mathcal{D}}$. There are only $K - 1$ coordinates that are \emph{not} one of the first $h$ corresponding to $x^{(1)}$. Therefore, within any non-redundant set of $K$ coordinates, at least one must be a non-redundant coordinate from the first $h$. This implies one of the desired trees is in the search space.
\end{proof}

\begin{proof}[Proof of \Cref{lem:top-k-fail-monotone}]
    Let $T$ be the tree returned by $\Top{k}$. Consider any root-to-leaf path of $T$ that does \emph{not} query any of the first $h$ coordinates (those within $x^{(1)}$). Recall that, with probability $(1 - \eps)$, the label is given by $\tribes_h(x^{(1)})$. On this path, the label of $T$ does not depend on any of the coordinates within $x^{(1)}$. Therefore,
    \begin{align*}
        \Prx_{(\bx,\by) \sim \mathcal{D}_{h,K}}[&T(\bx) = \by \mid \bx \text{ follows this path}] \\
        &= (1 - \eps)\cdot\Prx_{(\bx,\by) \sim \mathcal{D}_{h,K}}[T(\bx) = \tribes_h(\bx^{(1)}) \mid \bx \text{ follows this path}]  \\
        &\quad+ \eps\cdot \Prx_{(\bx,\by) \sim \mathcal{D}_{h,K}}[T(\bx) = \maj_K(\bx^{(2)}) \mid \bx \text{ follows this path}] \\
        & \leq (1-\eps) \cdot \left(\frac{1}{2} + o(1)\right) + \eps \cdot 1 \leq \frac{1 + \eps}{2} + o(1)
    \end{align*}
    where the last line follows because $\tribes_h$ is nearly balanced (\Cref{fact:tribes-properties}). As the distribution over $\bx$ is uniform, each leaf is equally likely. Therefore, if only $p$-fraction of root-to-leaf paths of $T$ query at least one of the first $h$ coordinates, then,
    \begin{equation*}
         \Prx_{(\bx,\by) \sim \mathcal{D}_{h,K}}[T(\bx) = \by] \leq (1- p)\cdot\left( \frac{1 + \eps}{2} + o(1)\right) + p \cdot 1 \leq \frac{1}{2} + \frac{p}{2} + \frac{\eps}{2} + o(1)
    \end{equation*}
    Our goal is to prove the tree returned by $\Top{k}$ achieves at most $\frac{1}{2} + \eps$ accuracy. Therefore, it is enough to prove that $p = o(1)$. Indeed, we will prove that $p \leq O(K^{-2})$.
    
    Here, we apply \citep[Lemma 7.4]{BLT-ITCS}, which was used to show that $\Top{1}$ fails to build a high accuracy tree. They used a different data distribution, but that particular Lemma still applies to our setting. They prove that a random root-to-leaf path of $T$ satisfies the following with probability at least $1 - O(K^{-2})$: If the length of this path is less than $O(K/\log K)$, at any point along that path, all coordinates within $x^{(2)}$ that have not already been queried have correlation at least $\frac{1}{100\sqrt{k}}$.
    
    That Lemma will be useful for proving $\Top{k}$ fails with the following parameter choices.
    \begin{enumerate}
        \item By setting $K \geq \Omega(h \log h)$, we can ensure all root-to-leaf paths in $T$ have length at most $O(K/\log K)$, so \citep[Lemma 7.4]{BLT-ITCS} applies. 
        \item By setting $K \leq O(h^2 / (\log h)^2)$, we can ensure that all the coordinates within $x^{(1)}$ have correlation less than $\frac{1}{100 \sqrt{k}}$ (\Cref{fact:tribes-properties}). This means that all non-redundant coordinates within $x^{(2)}$ have more correlation than those within $x^{(1)}$.
        \item By setting $k \leq K - h$, we ensure at all nodes along every path, there are at least $k$ coordinates within the last $K-1$ coordinates (those corresponding to $x^{(2)}$), that have not already been queried. With probability at least $1 - O(K^{-2})$ over a random path, those all have more correlation than all coordinates within $x^{(1)}$, so $\Top{k}$ won't query any of the $h$ coordinates within $x^{(1)}$.
    \end{enumerate}
    
    We conclude that, with probability at least $1 - O(K^{-2})$ over a random path in $T$, that path does not query any of the first $h$ variables. As a result, the accuracy of $T$ is at most $\frac{1 + \eps}{2} + o(1) \leq \frac{1}{2} + \eps$.
\end{proof}

%\newpage

\section{Details about datasets used in \Cref{sec:experiments}}
\label{appsec:dataset-details}

\begin{table}[H]
\centering
\begin{tabular}{ |c|c|c|c|c|c| } 
\hline
Name & Type & Size (\#train/\#test) & \#feats &\#binary feats &\#classes \\
\hline
\href{https://www.openml.org/d/40668}{connect-4} & C & 67557 (54045/13512) & 42 & 126 & 3 \\
\href{https://www.openml.org/d/26}{nursery} & C & 12960 (10368/2592) & 8 & 27 & 5 \\
\href{https://www.openml.org/d/6}{letter-recognition} & C & 19999 (15999/4000) & 16 & 256 & 26 \\
\href{https://www.openml.org/d/21}{car} & C & 1728 (1382/346) & 6 & 21 & 4 \\
\href{https://www.openml.org/d/3}{kr-vs-kp} & C & 3196 (2556/640) & 36 & 73 & 2 \\
\href{https://archive.ics.uci.edu/ml/datasets/HIV-1+protease+cleavage}{hiv-1-protease} & C & 6590 (5272/1318) & 8 & 160 & 2 \\
\href{https://www.openml.org/d/46}{molecular-biology-splice} & C & 3190 (2552/638) & 60 & 287 & 3\\
%\href{https://www.openml.org/d/24}{mushroom} & C & 8124 (6499/1625) & 22 & 117 & 2 \\
\href{https://www.openml.org/d/333}{monks-1} & C & 556 (444/112) & 6 & 17 & 2 \\
\href{https://www.openml.org/d/329}{hayes-roth} & C & 160 (128/32) & 4 & 15 & 3 \\
\href{https://www.openml.org/d/50}{tic-tac-toe} & C & 958 (766/192) & 9 & 27 & 2 \\
\href{https://www.openml.org/d/1459}{artificial-characters} & N & 10218 (8174/2044) & 7 & 91 & 10 \\
\href{https://www.openml.org/d/1120}{telescope} & N & 19020 (15216/3804) & 10 & 100 & 2 \\
\href{https://www.openml.org/d/44}{spambase} & N & 4601 (3680/921) & 57 & 57 & 2 \\
\href{https://archive.ics.uci.edu/ml/datasets/Dry+Bean+Dataset}{dry-bean} & N & 13611 (10888/2723) & 16 & 96 & 7 \\
\href{https://archive.ics.uci.edu/ml/datasets/Room+Occupancy+Estimation}{occupancy-estimation} & N & 10129 (8103/2026) & 16 & 86 & 4 \\
\href{https://archive.ics.uci.edu/ml/datasets/MiniBooNE+particle+identification}{miniboone} & N & 130064 (104051/26013) & 50 & 100 & 2 \\
\href{https://archive.ics.uci.edu/ml/datasets/dataset+for+sensorless+drive+diagnosis}{sensorless-drive-diagnosis} & N & 58509 (46807/11702) & 48 & 96 & 11 \\
\href{https://archive.ics.uci.edu/ml/datasets/First-order+theorem+proving}{ml-prove} & N & 6118 (4588/1530) & 51 & 51 & 6 \\
\href{https://archive.ics.uci.edu/ml/datasets/Avila}{avila} & N & 20867 (10430/10437) & 10 & 100 & 12 \\
\href{https://archive.ics.uci.edu/ml/datasets/Taiwanese+Bankruptcy+Prediction}{taiwanese-bankruptcy} & N & 6819 (5455/1364) & 95 & 95 & 2 \\
\href{https://archive.ics.uci.edu/ml/datasets/default+of+credit+card+clients}{credit-card} & N & 30000 (24000/6000) & 23 & 88 & 2 \\
\href{https://archive.ics.uci.edu/ml/datasets/Electrical+Grid+Stability+Simulated+Data+}{electrical-grid-stability} & N & 10000 (8000/2000) & 13 & 91 & 2 \\
\href{https://github.com/Jimmy-Lin/TreeBenchmark/blob/master/datasets/fico/data.csv}{FICO} & N & 1000 (900/100) & 23 & 1407 & 2 \\
\hline
\end{tabular}
\caption{Dataset characteristics. In the Type column, C stands for Categorial and N stands for Numerical.}
\label{table:dataset-details}
\end{table}

\Cref{table:dataset-details} provides complete details regarding all the datasets we used in our experiments. For datasets that do not provide an explicit train/test split, we randomly compute ten 80:20 splits, and average our results over these splits. The column \#feats has the number of raw attributes in each dataset, while the column \#binary feats has the number of features we obtain after converting these raw attributes to binary-valued attributes. For categorical datasets, we encode a categorical attribute taking on $l$ distinct values to $l$ binary attributes. For numerical datasets, we sort and compute thresholds for each numerical attribute. The number of thresholds is so selected that the total number of binary attributes does not exceed 100.

% \newpage
% \section{Test Accuracy comparison with Random-$k$}
% \label{appsec:comparison-with-random-k}

% \begin{figure}[H]
%     \centering
% 	\includegraphics[width=1\linewidth]{img/randomk-accuracy.pdf}  
% 	\caption{Test accuracy comparison between $\Top{1}$, $\Top{3}$, Random-$3$ and MurTree. We can see that $\Top{3}$ generally does better than Random-$3$ on the larger data sets, which justifies choosing features that maximize purity gain to split on. Interestingly, Random-$3$ outperforms $\Top{1}$ in certain cases.}
% 	\label{fig:randomk-accuracy}
% \end{figure}

%\newpage
\section{Implementation details for the $\topk$ algorithm}
\label{appsec:implementation-details}
Our implementation of Top-$k$ makes use of the DL8.5 algorithm implementation from \cite{aglin2021pydl8}. DL8.5 is an optimal classification tree search algorithm which utilizes caching and branch-and-bound optimization to avoid repeated computation and prune large sections of the search space that would yield suboptimal trees \cite{aglin2020learning}, similar to MurTree \cite{MURTREE}. To get our optimized Top-$k$ algorithm, we modify DL8.5 to only consider the first $k$ feature splits of each recursive state in descending order of information gain and with ties broken by feature index.

\begin{figure}[ht]
%   \captionsetup{width=.9\linewidth}
\begin{tcolorbox}[colback = white,arc=1mm, boxrule=0.25mm]
\vspace{3pt}
$\textnormal{Opt-Top-}k(\mathcal{H},S, h, ub)$:
\begin{itemize}[align=left]
	\item[\textbf{Given:}] A feature scoring function $\mathcal{H}$, a labeled sample set $S$ over $d$ dimensions, depth budget $h$, and upper bound on misclassification error $ub$.
	\item[\textbf{Output:}] Decision tree of depth $h$ that approximately fits $S$. 
\end{itemize}
\begin{enumerate}
	\item If $h = 0$, or if every point in $S$ has the same label, return the constant function with the best accuracy w.r.t.~$S$.
         \item \textbf{\textcolor{blue}{If $(S, d)$ is in the cache:}}
            \begin{enumerate}
                \item \textbf{\textcolor{blue}{Let $T_c$ and $ub_c$ be the cached tree and upper bound.}}
                \item \textbf{\textcolor{blue}{If $T_c \neq \textsc{NO-TREE}$  then return $T_c$.}}
                \item \textbf{\textcolor{blue}{If $T_c = \textsc{NO-TREE}$ and $ub \leq ub_c$ then return \textsc{NO-TREE}.}}
            \end{enumerate}
        \item \textbf{\textcolor{blue}{Let $T^*$ be \textsc{NO-TREE}.}}
        \item \textbf{\textcolor{blue}{Let $b^*$ be $ub + 1$.}}
	\item Let $\mathcal{I}\sse [d]$ be the set of $k$ coordinates maximizing $\mathcal{H}(S, i)$.
	\item For each $i \in \mathcal{I}$:
        \begin{enumerate}
            \item Let $T_i$ be the tree with 
                \begin{align*} 
            	\text{Root} &= x_i \\ 
            	\text{Left subtree} &= \textnormal{Opt-Top-}k(\mathcal{H},S_{x_i = 0},h - 1,b^* - 1)
        	\end{align*}
            \item \textbf{\textcolor{blue}{If the left subtree is \textsc{NO-TREE} then continue.}}
            \item \textbf{\textcolor{blue}{Let $b_L$ be the misclassification error of the left subtree w.r.t. $S_{x_i = 0}$.}}
            \item \textbf{\textcolor{blue}{If $b_L \leq b^*$}} we define the right subtree of $T_i$
                \begin{align*}
                    \text{Right subtree} &= \textnormal{Opt-Top-}k(\mathcal{H},S_{x_i = 1},h - 1,b^* - 1 - b_L)
                \end{align*}
            \item \textbf{\textcolor{blue}{If the right subtree is \textsc{NO-TREE} then continue.}}
            \item \textbf{\textcolor{blue}{Let $b_R$ be the misclassification error of the right subtree w.r.t. $S_{x_i = 1}$.}}
            \item \textbf{\textcolor{blue}{If $b_L + b_R < b^*$:}}
                \begin{enumerate}
                    \item \textbf{\textcolor{blue}{Let $T^* = T_i$.}}
                    \item \textbf{\textcolor{blue}{Let $b^* = b_L + b_R$.}}
                \end{enumerate}
            \item \textbf{\textcolor{blue}{If $b_L + b_R = 0$ then break.}}
        \end{enumerate}
	\item \textbf{\textcolor{blue}{Add $(S, d)$ to the cache with value $(T^*, ub)$.}}
	\item Return $T^*$.
\end{enumerate}
\end{tcolorbox}
\caption{The optimized Top-$k$ algorithm is equivalent to the Top-$k$ algorithm described in \Cref{fig:pseudocode} but with caching and pruning optimizations that make it significantly faster in practice. These changes are bolded and highlighted in blue.
}
\label{fig:pseudocode-optimized}
\end{figure}

There were two other optimizations made by the DL8.5 algorithm implementation that would have led to different results. These optimizations are (1) fast computation of depth-two optimal trees and (2) similarity-based lower bounding. These optimizations were disabled.

%\newpage
\section{Training time comparison}
\label{appsec:training-time-comparison-top-1}

\begin{figure}[H]
	\begin{subfigure}{.24\textwidth}
    	\centering
    	\includegraphics[width=1\linewidth]{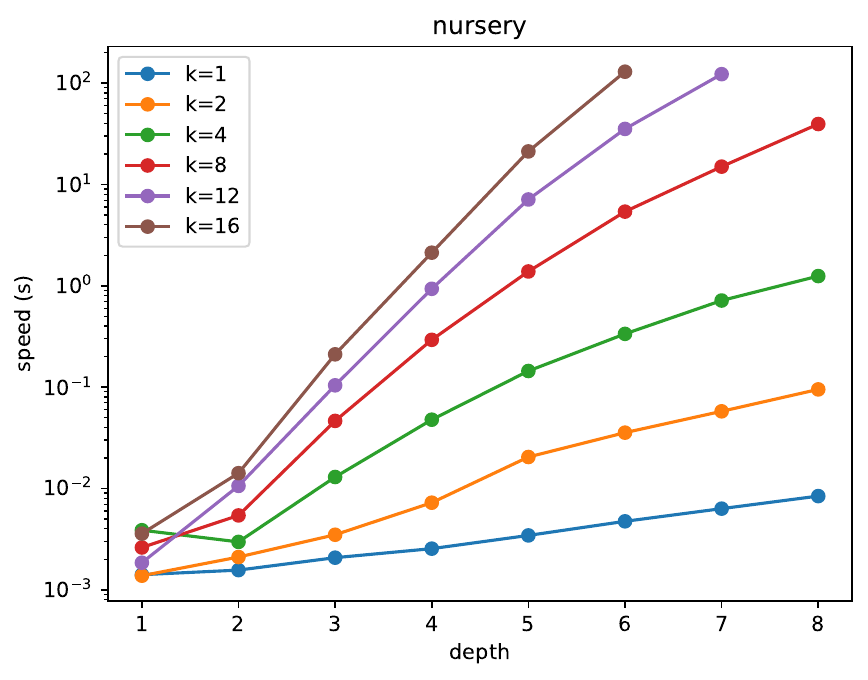}  
        \end{subfigure}
	\begin{subfigure}{.24\textwidth}
    	\centering
    	\includegraphics[width=1\linewidth]{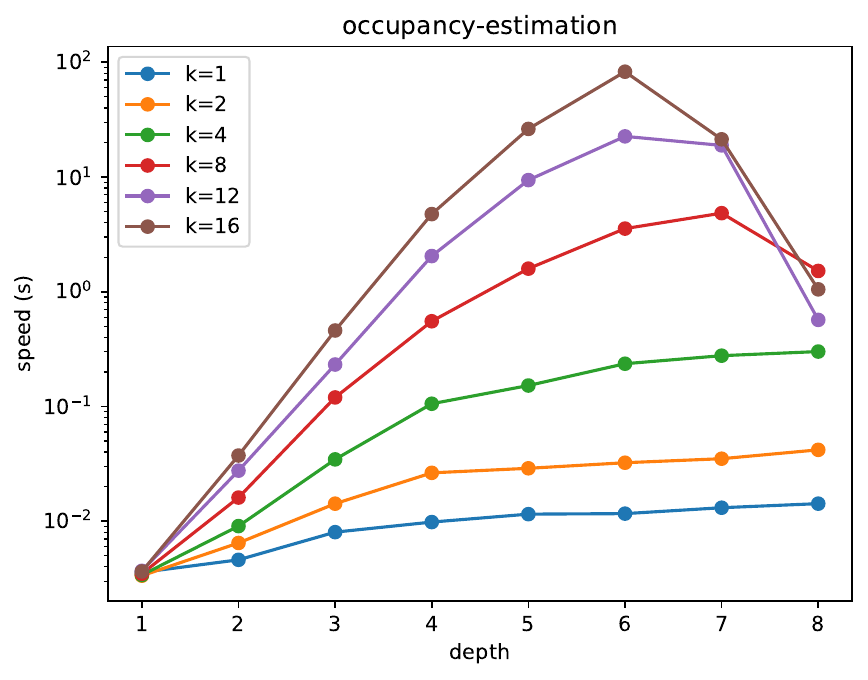}  
        \end{subfigure}
	\begin{subfigure}{.24\textwidth}
    	\centering
    	\includegraphics[width=1\linewidth]{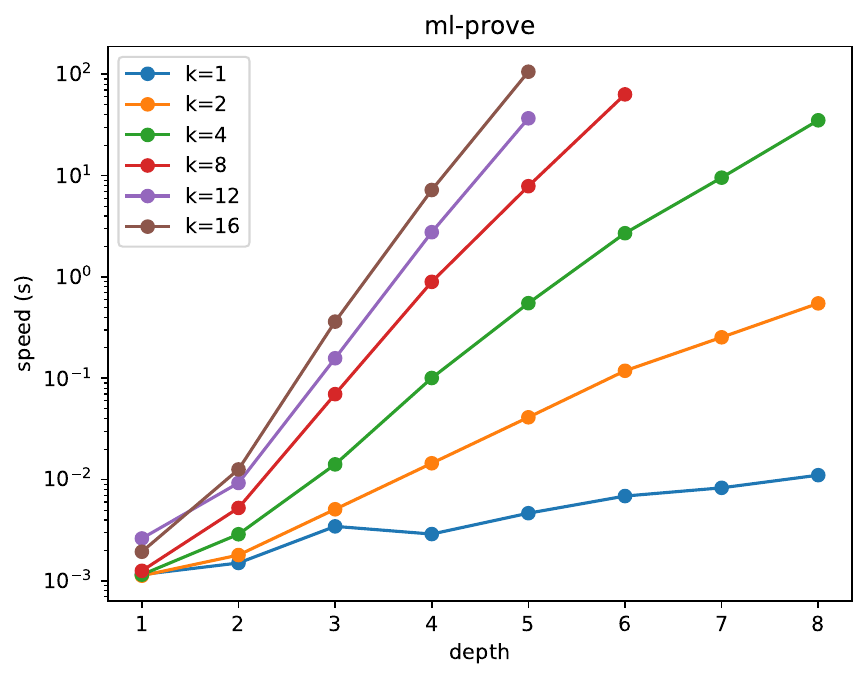}  
        \end{subfigure}
	\begin{subfigure}{.24\textwidth}
    	\centering
    	\includegraphics[width=1\linewidth]{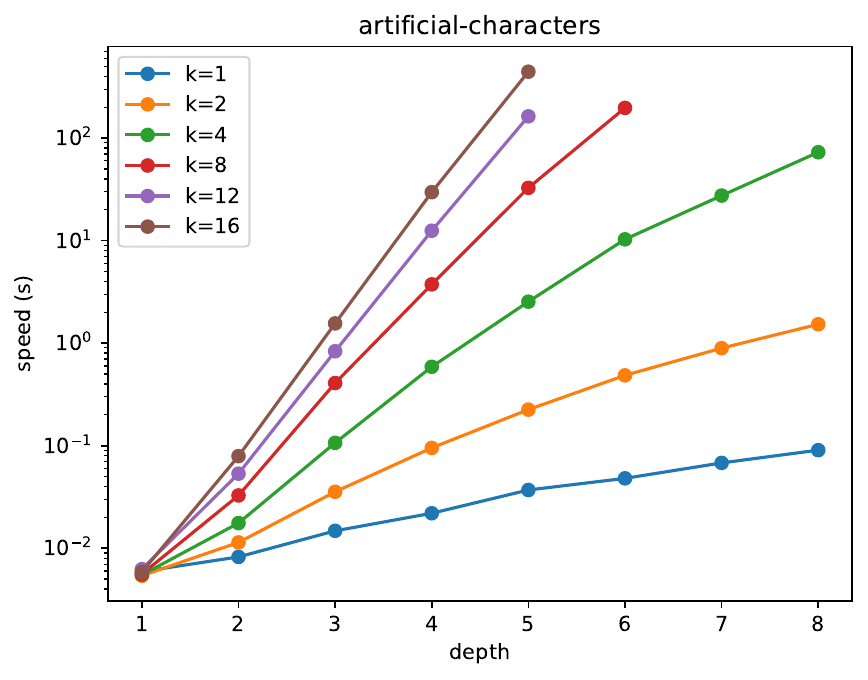}  
        \end{subfigure}
        \newline
	\begin{subfigure}{.24\textwidth}
    	\centering
    	\includegraphics[width=1\linewidth]{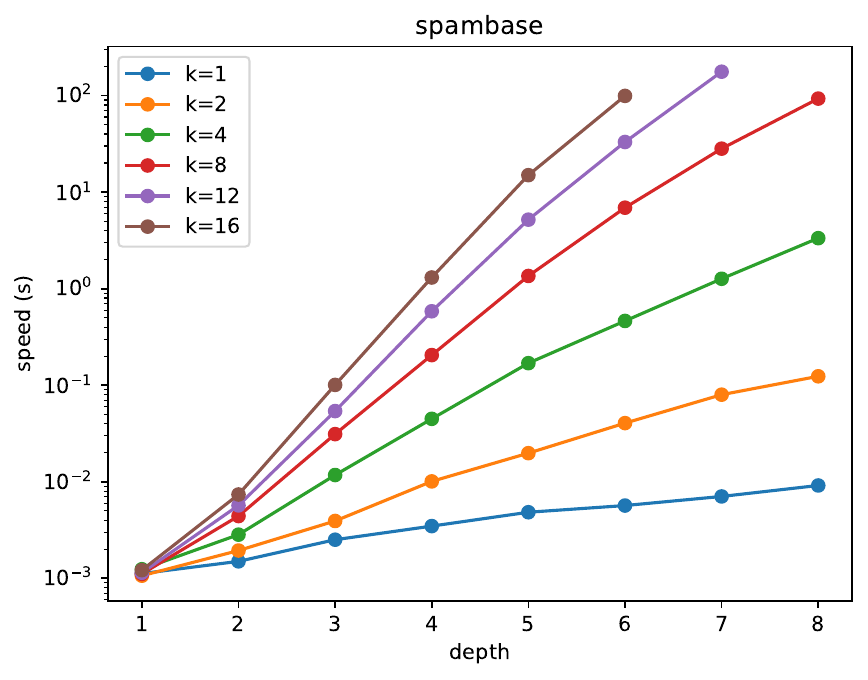}  
        \end{subfigure}
	\begin{subfigure}{.24\textwidth}
    	\centering
    	\includegraphics[width=1\linewidth]{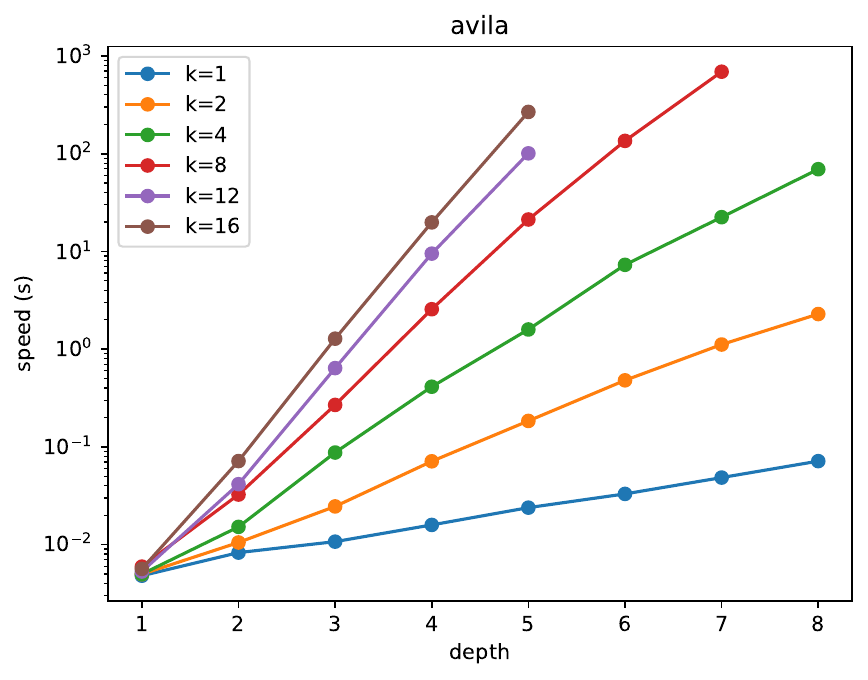}  
        \end{subfigure}
	\begin{subfigure}{.24\textwidth}
    	\centering
    	\includegraphics[width=1\linewidth]{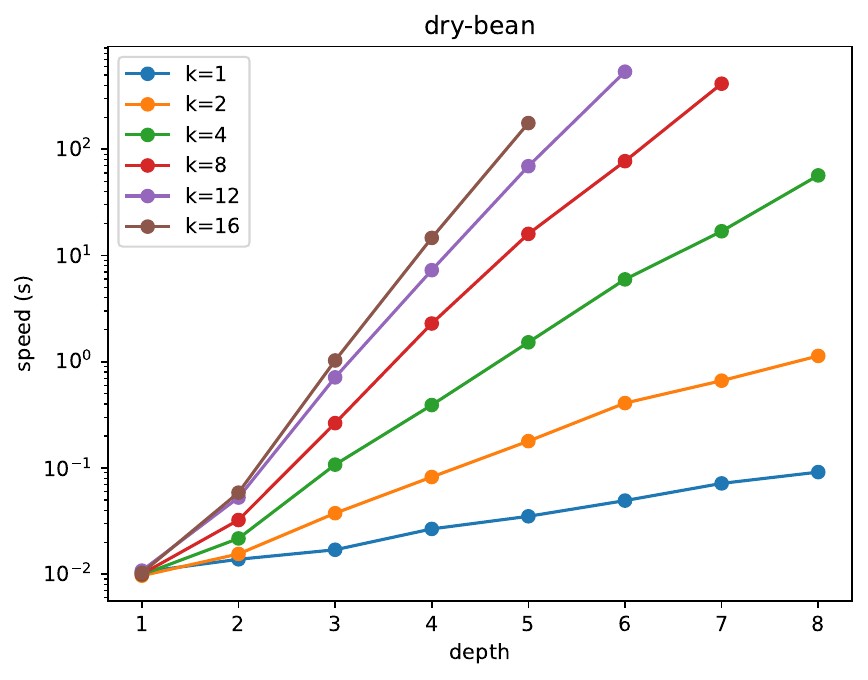}  
        \end{subfigure}
	\begin{subfigure}{.24\textwidth}
    	\centering
    	\includegraphics[width=1\linewidth]{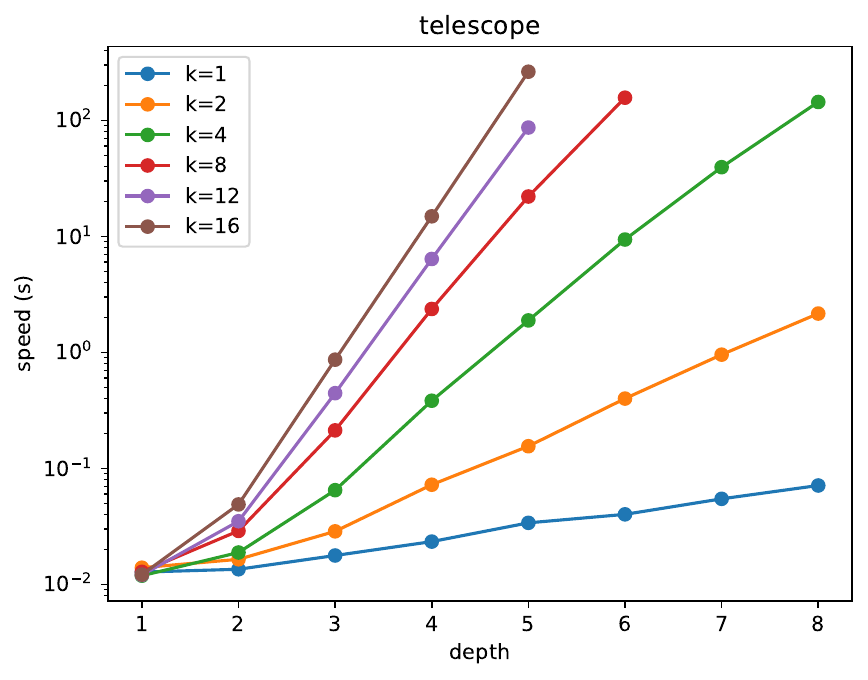}  
        \end{subfigure}
        \newline
	\begin{subfigure}{.24\textwidth}
    	\centering
    	\includegraphics[width=1\linewidth]{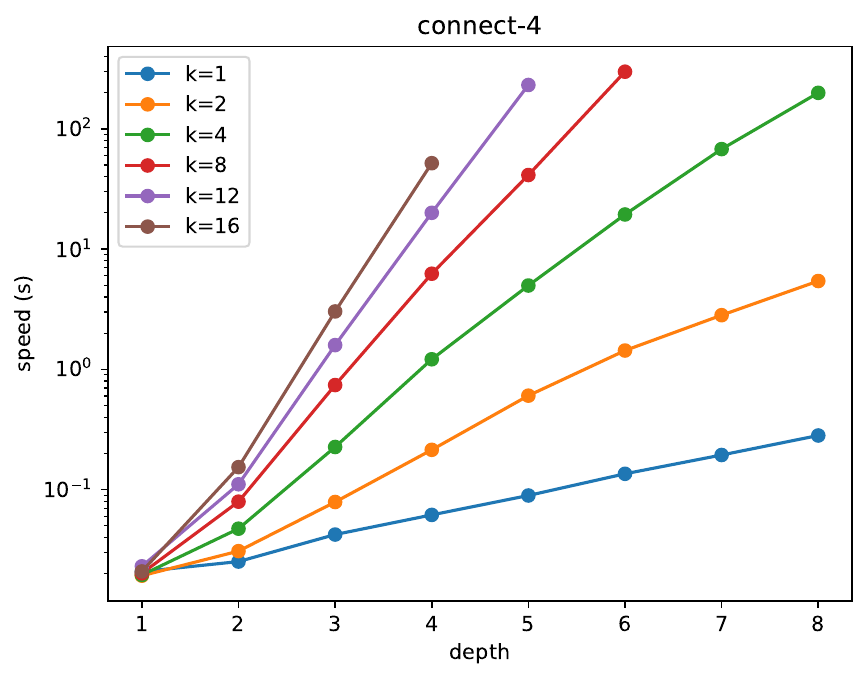}  
        \end{subfigure}
	\begin{subfigure}{.24\textwidth}
    	\centering
    	\includegraphics[width=1\linewidth]{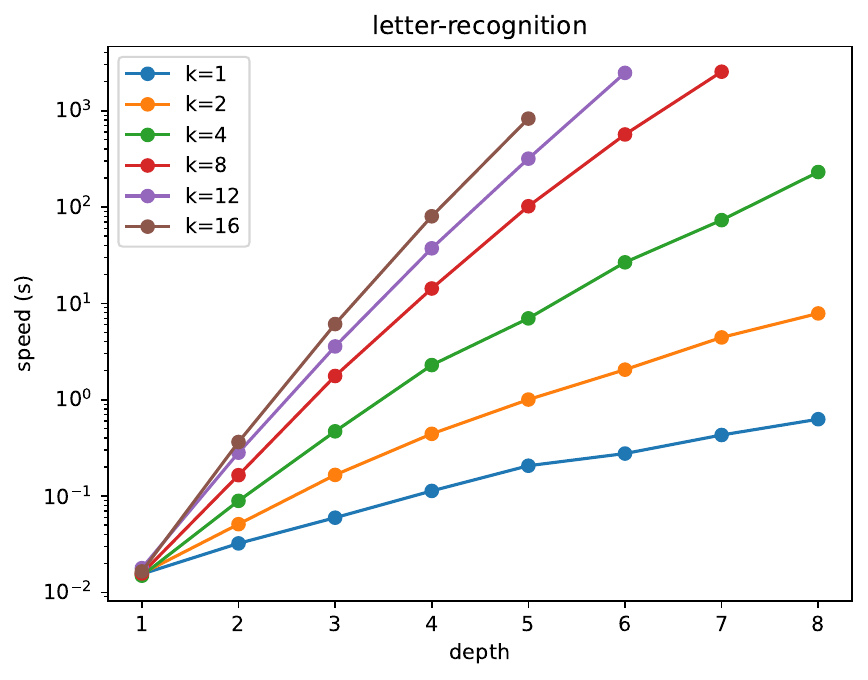}  
        \end{subfigure}
	\begin{subfigure}{.24\textwidth}
    	\centering
    	\includegraphics[width=1\linewidth]{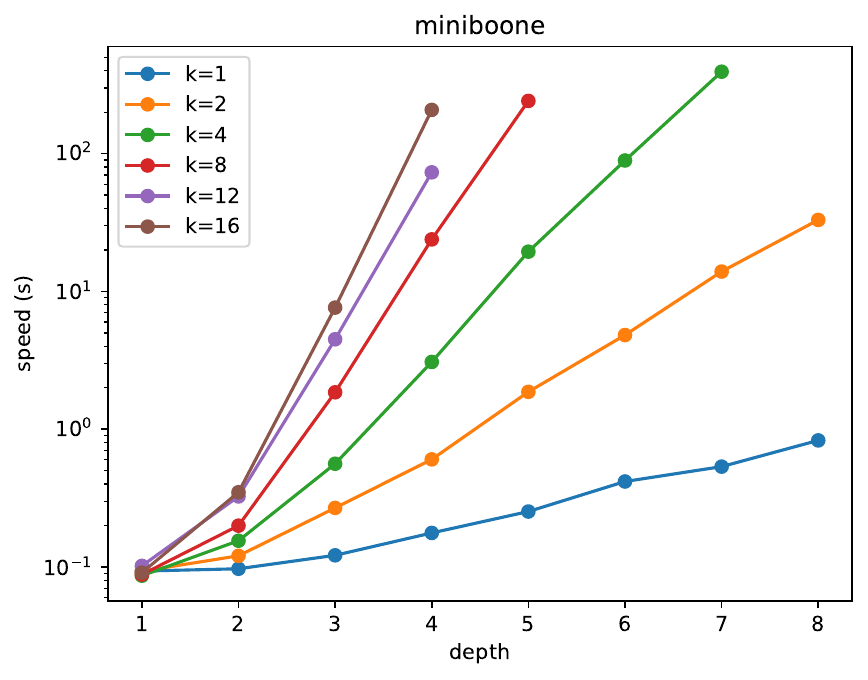}  
        \end{subfigure}
	\begin{subfigure}{.24\textwidth}
    	\centering
    	\includegraphics[width=1\linewidth]{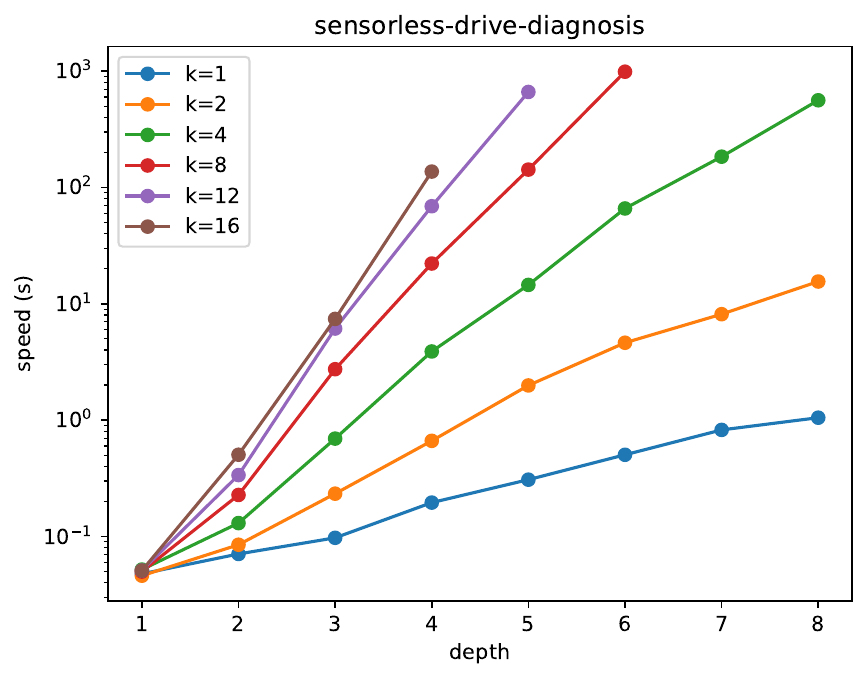}  
        \end{subfigure}
	\caption{Training time comparison between $\Top{1}$ and $\topk$. We can see that the blowup in training time when compared to $\Top{1}$ is relatively mild. %\textcolor{red}{\textbf{colin: }\textit{can we say exponential (and sometimes sub-exponential), as expected}?}.
    In particular, for $k=2$, we are able to go all the way up until depth-8 trees within 1 second in almost all cases. Even $k=4,8$ finishes execution for depth-5 trees within $\approx$ 20 seconds for majority of the datasets. Interestingly, in the case of occupancy-estimation, we can see that the training times get \textit{faster} at the larger depths. This is an artefact of the optimized branch-and-bound implementation of DL8.5, which stops branching once it discovers a subtree with no errors. We expect these perfect subtrees to become more prevalent when considering higher depth trees and when there are fewer points to be classified. }
	\label{fig:time-topk-vs-top1}
\end{figure}

% \begin{figure}[H]
%     \centering
% 	\includegraphics[width=1\linewidth]{img/timing-topk.pdf}  
% 	\caption{Training time comparison between $\Top{1}$, $\topk$ and MurTree. We can see that the plots for $\topk$ are all straight lines with increasing slope, as would be expected from \Cref{claim:running-time}. The dashed vertical line for MurTree indicates a segmentation fault before completing execution of building the decision tree.}
% 	\label{fig:time-topk-vs-top1}
% \end{figure}

%\newpage
\section{Accuracy comparison with $\Top{1}$ -- further plots}
\label{appsec:accuracy-further-plots}

\begin{figure}[H]
    \centering
	\includegraphics[width=1\linewidth]{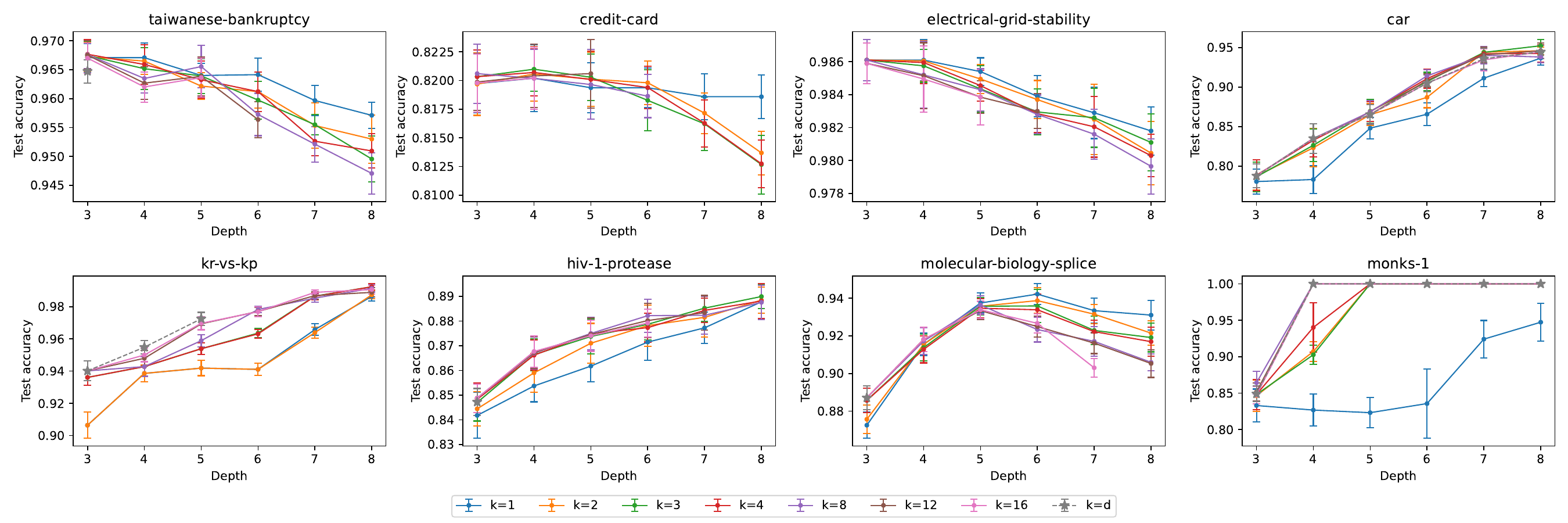}  
	\caption{Test accuracy comparison between $\Top{1}$ and $\topk$.}
	\label{fig:app-acc-topk-vs-top1}
\end{figure}

We provide plots from our experiments on a further few datasets comparing the test accuracy of $\topk$ and $\Top{1}$ in \Cref{fig:app-acc-topk-vs-top1}. In the case of taiwanese-bankruptcy, credit-card and electrical-grid-stability, we can observe that $\Top{1}$ is outperforming $\topk$. However, we believe that this is because the learning problem in this regime is extremely susceptible to overfitting. In particular, we can see that $\Top{1}$ is itself not consistently improving with increasing depth. Concretely, increasing depth beyond 3 is already causing $\Top{1}$ to overfit, and hence we would expect $\topk$ to suffer from overfitting even more. 
%Furthermore, we can see that the gradation in the y-axis is very small, in that the accuracy numbers are very close to one another. 
In the case of the remaining datasets (which all happen to be categorical), while the numbers might not be monotonically getting better with increasing $k$, we can still observe that there is always some value of $k > 1$ which is outperforming $k=1$ (except for molecular-biology-splice, for which this is still the case till depth 6). This lends further support to our proposition of incorporating $k$ as an additional hyperparameter to tune while training decision trees greedily.

\end{document}